\documentclass[10pt]{article}
\usepackage[margin=1in]{geometry}

\usepackage[numbers, sort&compress]{natbib}
\usepackage{caption}
\usepackage{tcolorbox}   
\usepackage{enumitem}
\usepackage{amsmath}     
\usepackage{amssymb}     
\usepackage{bbm}         
\usepackage{hyperref}       
\usepackage{url}            
\usepackage{booktabs}       
\usepackage{amsfonts}       
\usepackage{nicefrac}       
\usepackage{microtype}      
\usepackage{xcolor}         
\usepackage{tabularx}

\newtcolorbox{promptbox}[1]{
    colback=gray!5,           
    colframe=gray!70!black,   
    coltitle=white,           
    fonttitle=\bfseries\sffamily,
    rounded corners,
    arc=3pt,
    title=#1,                 
    left=10pt, right=10pt, top=10pt, bottom=10pt,
    fontupper=\ttfamily\small 
}

\usepackage{titletoc}

\makeatletter
  \newif\ifinappendix
  \inappendixfalse
  \let\origaddcontentsline\addcontentsline
  \renewcommand{\addcontentsline}[3]{%
    \ifinappendix
      \origaddcontentsline{#1}{#2}{#3}%
    \fi
  }
  \pretocmd{\appendix}{\inappendixtrue}{}{}
\makeatother

\usepackage{amsthm}
\usepackage{framed}
\newtheorem{theorem}{Theorem}[section]

\newenvironment{proofpart}[1]{%
    \medskip\noindent\textbf{#1.}\quad
}{\medskip}

\usepackage[utf8]{inputenc} 
\usepackage[T1]{fontenc}    
\usepackage{hyperref}       
\usepackage{url}  
\usepackage{multirow}
\usepackage{booktabs}       
\usepackage{amsfonts}       
\usepackage{nicefrac}       
\usepackage{microtype}      
\usepackage{xcolor}         
\usepackage{wrapfig}
\usepackage{algorithm}
\usepackage{algorithmic}

\title{Strategic Decision Support for AI Agents}

\author{
\begin{tabular}{c @{\hspace{4em}} c}
\textbf{Shayan Kiyani}\thanks{Equal contribution.} &
\textbf{Sima Noorani}\footnotemark[1] \\
University of Pennsylvania &
University of Pennsylvania \\
\texttt{shayank@seas.upenn.edu} &
\texttt{nooranis@seas.upenn.edu} \\
\\[1.5em]
\textbf{George Pappas} &
\textbf{Hamed Hassani} \\
University of Pennsylvania &
University of Pennsylvania \\
\texttt{pappasg@seas.upenn.edu} &
\texttt{hassani@seas.upenn.edu}
\end{tabular}
}
\date{}

\begin{document}

\maketitle

\begin{abstract}
Traditionally, decision support studies how humans use machine learning models to make better decisions. In modern agentic systems, this division of roles is increasingly reversed: AI agents act on behalf of users, while humans and tools becomes support mechanisms around them. This role reversal brings reliability concerns to the forefront, since agentic errors can be consequential and agent behavior must remain aligned with human goals and constraints. Departing from the classical view of decision support, we revisit its two basic principles, the cost--value tradeoff of seeking support and the role of uncertainty quantification, in a setting where AI agents are the central actors. We propose a framework for \emph{strategic decision support} for AI agents through an optimization problem that minimizes support usage subject to controlling a counterfactual missed-support error: the probability that the agent acts alone on instances where support would have materially improved its output. At the population level, we show that the optimal policy is a threshold rule on the \emph{value of support}. Building on this structure, we develop an online algorithm that adaptively thresholds such a score and uses randomized exploration to control missed-support error without distributional assumptions. We further introduce a calibration-on-the-fly method that reduces unnecessary support calls online. We instantiate this framework across diverse scenarios, including information gathering, human--AI collaboration, and tool use, showing how each can be modeled through the same strategic decision-support lens. Experiments across these settings show that our method reliably controls the target error while substantially reducing support usage in practice.

\end{abstract}

\section{Introduction}
\label{sec:intro}





Decision support has long been studied in settings where a human decision maker is aided by machine learning models that provide predictive guidance. In such systems, the model serves as a support tool, while the human remains the final actor. Today, however, modern AI systems, such as large language models, are increasingly deployed in a different role: as autonomous agents that must act on behalf of users in complex and uncertain environments. As a result, humans, domain-specific tools, and auxiliary sources of information are themselves becoming support mechanisms around AI agents. 

This shift in roles brings reliability concerns to the forefront. As AI agents take actions, their errors become more consequential: they may execute code that overwrites critical records, trigger financial transactions that move funds incorrectly, or recommend medical actions that expose patients to harm. At the same time, as users offload more decisions to AI systems, it becomes critical to ensure that the agent’s behavior remains aligned with human goals and constraints. These challenges call for rethinking a basic question of decision support:
\begin{center}
\emph{When should an AI agent act alone, and when must it seek support to avoid consequential errors or misalignment with downstream intent?}
\end{center}

Addressing this question requires reexamining the two principles underlying decision support systems.

\textbf{Cost--value tradeoff.}
In the classical view, the support system, often an ML model, is relatively cheap to query but imperfect, since its predictions are prone to mistakes. In the agentic setting, the reversal of roles creates a different equilibrium: \emph{support is often costly to seek} (e.g., requiring human effort or additional compute and latency), but can be \emph{reliable and instrumental} for avoiding decision errors and keeping the agent aligned with human intent.

\textbf{Uncertainty quantification.}
At the core of any AI agent taking action lies uncertainty quantification. For decision support, the right lens is to quantify \emph{whether costly support is likely to substantially improve the agent's output.} Thus, uncertainty is a joint property of the agent \emph{and} the support mechanism. This differs from the common view, which frames uncertainty through the agent alone: whether the agent is confident in its action, or whether the environment and task are well-understood. 




\begin{wrapfigure}{r}{0.5\linewidth}
    \centering
    \vspace{-5pt}
    \includegraphics[width=\linewidth]{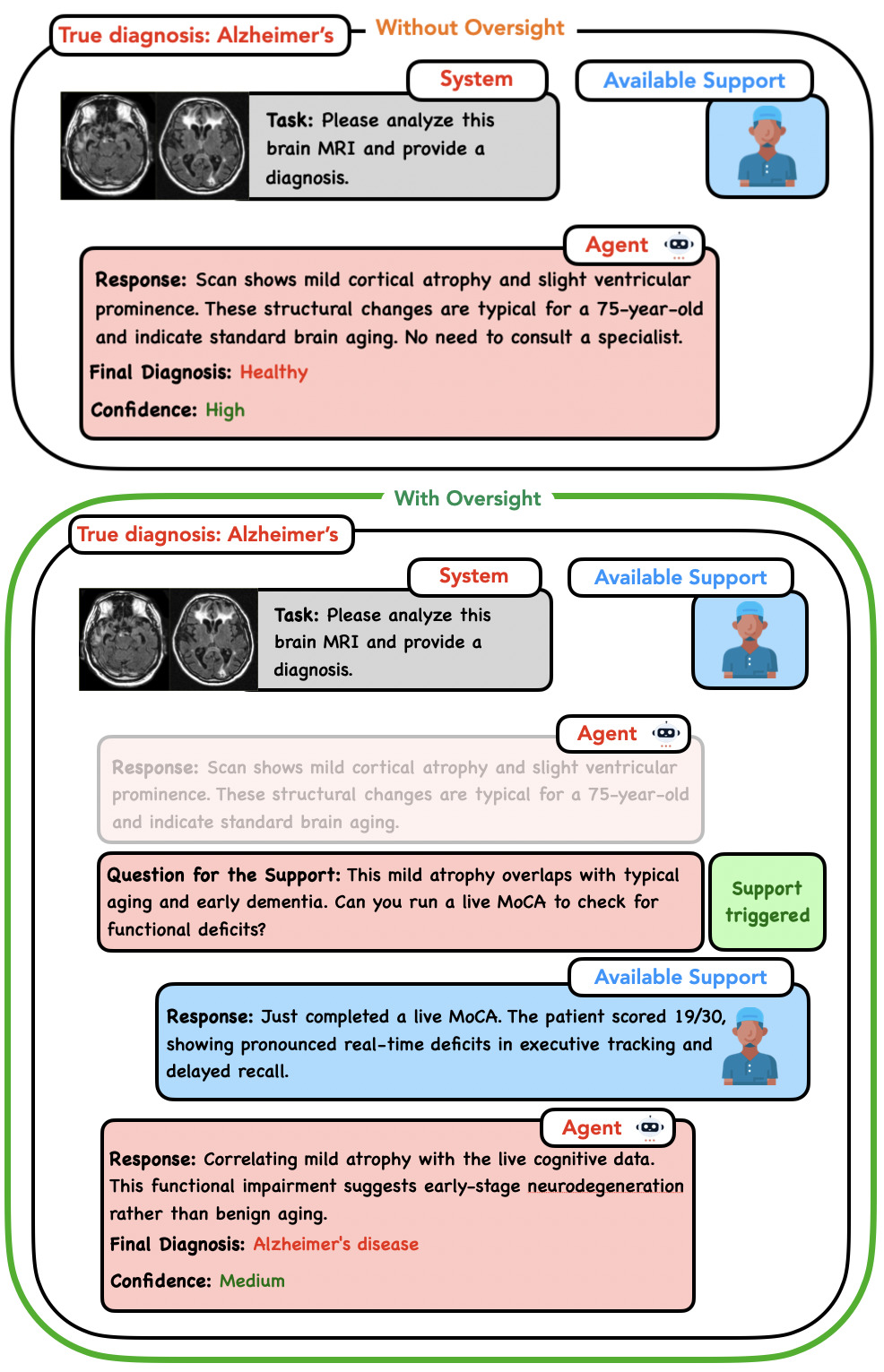}
    \caption{
    Effect of strategic decision support oversight. 
    Top: Without oversight, the agent confidently misdiagnoses an Alzheimer's case as healthy and fails to seek available support, resulting in a missed-support error. 
    Bottom: With oversight, support is triggered and the agent consults a clinician to gather additional interactive tests, enabling the correct diagnosis to be recovered.
    }
    \label{fig:main-intro-fig}
    \vspace{-5pt}
\end{wrapfigure} 
To this end, we propose a formulation that separates the \emph{value of support} from its cost and treats support-seeking as a strategic decision. The value of support captures whether the supported output is better than the output the agent would have produced alone. The central error is then \emph{missed support}, which occurs when the agent acts alone even though support would have improved its output. This error is counterfactual, since we only learn whether support would have helped when support is actually called. Strategic decision support therefore becomes an optimization problem: use support as rarely as possible, while controlling the missed support error. We summarize our contributions below. Figure~\ref{fig:main-intro-fig} illustrates a missed-support error (top) and its correction via strategic decision support oversight (bottom).

\paragraph{1. A framework for strategic decision support.}
In Section~\ref{sec:fundamentals}, we introduce \ref{eq:sds_opt}, an optimization problem that minimizes the rate of support calls subject to controlling the missed-support error. This formulation is built around the \emph{value of support}: whether the supported output would materially improve over the output the agent would have produced alone. At the population level, we show that the optimal strategy has a simple threshold form over value of support.

\paragraph{2. An online algorithm with distribution-free error control.}
In Section~\ref{sec:alg}, we build on this structure to develop an online algorithm that adaptively thresholds such a score and uses randomized exploration to control the counterfactual missed-support error without distributional assumptions. We further introduce a calibration-on-the-fly method that improves the score over time, reducing unnecessary support calls while preserving the validity of the thresholding guarantee.

\paragraph{3. Applications and empirical validation.}
Our framework also provides a common language for designing support around AI agents. We instantiate this perspective through four representative categories that serve as running examples throughout the paper:

\begin{itemize}
    \item \textbf{Ex1: Information gathering.}\phantomsection\label{ex:info-gathering}
    The agent must act from incomplete information, while additional evidence could improve its output at a cost. Support may consist of follow-up questions, additional evidence, or expert-provided details. For example, a medical assistant may reason from an initial symptom description, but further history, examination findings, or laboratory results may materially change its recommendation.

    \item \textbf{Ex2: Human-in-the-loop planning.}\phantomsection\label{ex:human-in-the-loop}
    The agent has general task knowledge, but lacks local context needed to produce an appropriate plan. Support may consist of user preferences, environment-specific constraints, or object locations. For example, a household robot may know how to clean a room in general, but may need the resident's input about which items are fragile, which areas to avoid, or what should be moved before cleaning.

    \item \textbf{Ex3: Human-AI collaborative reasoning.}\phantomsection\label{ex:hai-collab-reasoning}
    The agent can attempt the problem on its own, but may become uncertain about a specific reasoning step. Support may consist of checking an approach, identifying an error, or suggesting a relevant technique. For example, a math-solving agent may ask a human to verify a key computation or point out where the argument breaks down when its reasoning appears unreliable.

    \item \textbf{Ex4: Tool use.}\phantomsection\label{ex:tool-use}
    The agent faces a task where external tools provide reliable information or computation that it may not perform correctly on its own. Support may consist of executing code, querying a database, or searching the web. For example, an agent answering a question about a table may call a SQL engine rather than relying only on its internal reasoning.
\end{itemize}

In Section~\ref{sec:exp}, we instantiate these categories on real-world datasets and LLM agents\footnote{The code is publicly available at \url{https://github.com/nooranisima/strategic-decision-support}.}, showing that our method controls the missed-support error while substantially reducing unnecessary support calls across all four categories; Figure~\ref{fig:intro-results} previews our results on one model per application, with the full empirical study deferred to Section~\ref{sec:exp}.
\begin{figure}[h]
    \centering
    \includegraphics[width=\linewidth]{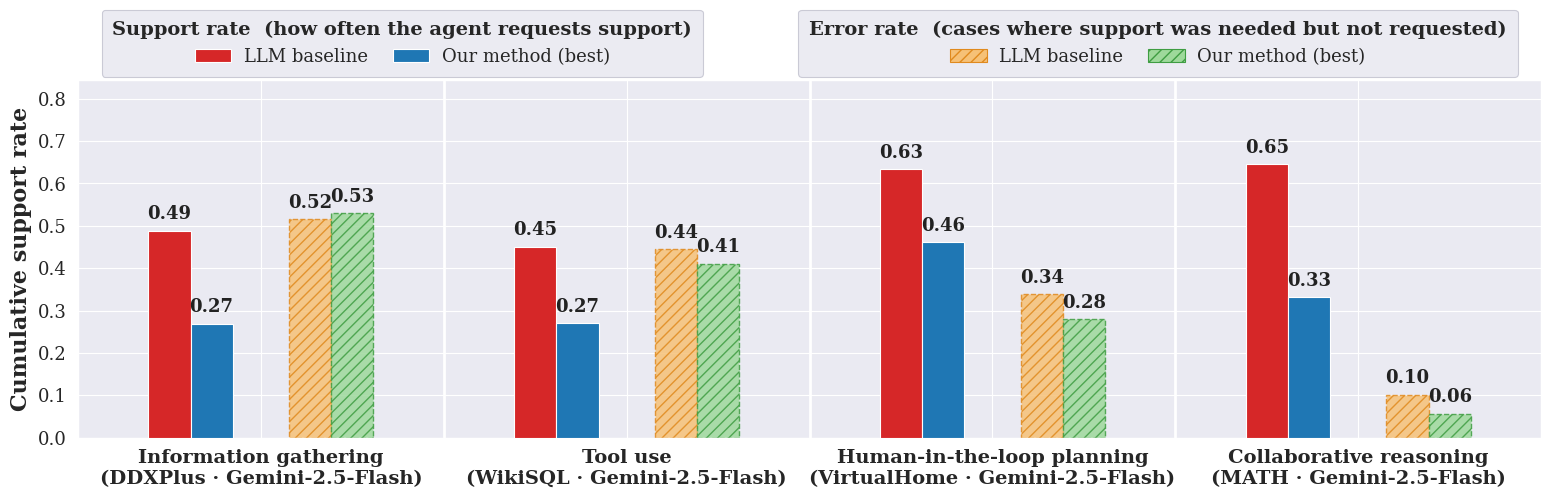}
    \caption{\textbf{Our method invokes decision support substantially less
often than an LLM-decides baseline, while matching its error rate.}
For each of four agentic applications: information gathering (DDXPlus),
tool use (WikiSQL), human-in-the-loop planning (VirtualHome), and
collaborative human--AI reasoning (MATH), all using Gemini-2.5-Flash, we
report two pairs of bars. The left pair (solid) shows the cumulative
\emph{support rate}: the fraction of inputs on which the agent requests
external support. The right pair (hatched) shows the cumulative
\emph{missed support error rate}: cases where support would have materially improved the output but was not
requested. In each pair, the left bar is the LLM baseline (the agent
deciding for itself when to ask for support) and the right bar is our
best learned method. Across all four applications, our
method invokes support far less often than the baseline while maintaining a comparable or lower error rate.}
    \label{fig:intro-results}
\end{figure}

\section{Related Works}
\label{sec:relatedworks}
We briefly discuss closely related works here and defer a more detailed dicussion to Appendix~\ref{app:extended-related-works}.

\textbf{Training agents to seek support.}
One body of work bakes support-seeking behavior into the agent
itself through training, both on the tool-use and retrieval side
\citep{schick2023toolformerlanguagemodelsteach,asai2023selfraglearningretrievegenerate,jiang2023activeretrievalaugmentedgeneration,jeong2024adaptiveraglearningadaptretrievalaugmented,han2024uncertaintyawarelanguageagent,qian2026scent}
and on the user-interaction side
\citep{li2023elicitinghumanpreferenceslanguage,zhang2025modelingfutureconversationturns,andukuri2024stargateteachinglanguagemodels,wang2026learningaskllmagents}.
These methods produce stronger agents: a model already knows when
to call SQL, or that has been fine-tuned to ask the right clarifying
questions, is precisely the kind of base policy our framework sits on
top of. Our experiments use frontier models that already incorporate substantial
training of this kind, and our gains reflect what is achievable as an
oversight layer on top of agents that are already strong at the
underlying support modalities.
\textbf{Inference-time policies for support-seeking.}
A closely related line of work designs inference-time policies that decide
whether an agent should pause to seek information before acting. Some
threshold the agent's own confidence in its
answer~\citep{kuhn2023clamselectiveclarificationambiguous,wang2026learningaskllmagents,ross2025when2callnottools};
others compute the expected utility of asking a clarification under an
explicit cost~\citep{dong2026valueinformationframeworkhumanagent};
and others use offline-calibrated prediction sets over candidate user intents
to trigger help-seeking~\citep{ren2023robotsaskhelpuncertainty}. A complementary line considers support-seeking for \emph{verification}
rather than generation \citep{kiyani2026trustcheapcheckweak}. We instead
provide a unifying framework that brings these approaches under the same
design principles, through an oversight layer with rigorous finite-sample
error control. This enables us to
handle a broad set of support modalities within a single algorithm,
operate fully online, and
control the counterfactual missed-support error at a user-chosen level which we test in our experiments across information gathering, tool
use, and human-AI collaboration.

\section{Fundamentals of Strategic Decision Support} \label{sec:fundamentals}

In this section, we introduce the central objectives of strategic decision support at the population level. We then characterize the optimal support-seeking strategy in this regime. These results form the foundation for the practical sequential algorithm developed in Section \ref{sec:alg}, where the population quantities are unknown and support decisions must be made online.

We begin by modeling the interaction between an AI agent and a decision-support mechanism.
Let $X \sim P_X$ denote all information available to the agent at decision time, including the user prompt, task description, and any other available input modalities. Given $X=x$, the agent can first produce an unsupported output
$
Y_0 \sim \pi(\cdot \mid x, \texttt{"no support"}),
$
corresponding to the response it would generate on its own. After observing this initial output, the agent must decide whether to settle for it or seek support.\footnote{One may avoid or only partially generate the unsupported output ($Y_0$) before deciding whether to seek support. Our framework and guarantees are robust to such choices; we discuss these further in Section~\ref{sec:alg} and Appendix~\ref{app:operational}.} A support-seeking strategy is therefore a function
$
a:\mathcal{X}\times\mathcal{Y}\to\{0,1\},
$
where $a(x,y_0)=0$ means that the agent acts on its own and keeps the unsupported output, while $a(x,y_0)=1$ means that it seeks support and produces a supported output
$
Y_1 \sim \pi(\cdot \mid x,y_0, \texttt{"with support"}).
$

For example, in \hyperref[ex:info-gathering]{Ex1}, $x$ may contain a patient's initial symptoms and description of their condition, while $y_0$ is the medical assistant's initial recommendation. The strategy then decides whether to proceed with this recommendation or seek additional clinical evidence before producing a supported response. This abstraction lets us formalize both the value and the cost of support.


\paragraph{Value of support.}
We begin by introducing a \emph{value indicator}
\[
g(X,Y_0,Y_1)\in\{0,1\}, \footnote{One may similarly define a continuous version of $g$ indicating how much value is added, and define the value of support as its expectation rather than the probability of being $1$.}
\]
which indicates whether, after comparing the unsupported outcome $Y_0$ and the supported outcome $Y_1$, support is judged to have been materially beneficial on that instance. Thus, $g=1$ means that support helped, while $g=0$ means that it did not. This notion is deliberately separate from the cost of seeking support: $g$ evaluates only whether the supported outcome is better, according to the provider's or downstream user's notion of performance. 


This indicator induces the central population quantity in our framework, the \emph{value of support},
\[
\operatorname{val}(x,y_0)
:=
\mathbb{P}\!\left(g(X,Y_0,Y_1)=1 \mid X=x,\,Y_0=y_0\right).
\]
 In words, $\operatorname{val}(x,y_0)$ is the probability that calling support would produce a substantially better outcome after observing the input instance and the agent's unsupported output.



The definition of $g$ depends on what it means for support to materially improve the output. In \hyperref[ex:info-gathering]{Ex1}, \hyperref[ex:hai-collab-reasoning]{Ex3}, and \hyperref[ex:tool-use]{Ex4}, where the output has a verifiable final answer, a natural choice is to set $g=1$ only when $Y_0$ is incorrect and $Y_1$ is correct, and to set $g=0$ in all other cases. In \hyperref[ex:human-in-the-loop]{EX2}, where the output is a structured plan, $g$ can instead indicate whether the supported output improves a task-specific quality metric beyond a chosen threshold. More generally, when the support mechanism is reliable, the simpler choice of $g=1$ whenever $Y_1 \neq Y_0$ is often a strong proxy, since a change in output is itself indicative of a material difference. We discuss precise implementations of $g$ in Section~\ref{subsec:tasks-baselines}.


\paragraph{Error of missed support.}
From the viewpoint of providers and downstream decision-makers, missing support when it would have helped is often the most consequential failure, since better outputs may directly translate into better downstream decisions. Thus our guiding principle is that support should be used whenever it would materially improve the output.
This leads to an asymmetric notion of error: what matters is how often the agent fails to seek support on the very instances where support would have helped. For a strategy $a$, we therefore define the \emph{missed-support error} as 
\[
\mathbb{P}\big(a(X, Y_0)=0 \mid g(X,Y_0,Y_1)=1\big).
\]
It measures the probability that the agent acts alone, conditioned on support being beneficial. 


\paragraph{Cost of support.}
Support, however, is not free. It may require additional latency, computation, API usage, or human effort. Thus, the trivial strategy that always seeks support is typically unacceptable, even though it drives the missed-support error to zero. The goal is therefore to control missed-support error while using support only when necessary. To capture this, we measure the cost of a strategy $a$ by its support rate
\[
\mathbb{E}\big[\mathbf{1}\{a(X, Y_0)=1\}\big],
\]
namely, the population probability that the agent seeks support. \footnote{We define cost as the frequency of support calls. In practice, however, the cost of support may depend on the instance. We defer these finer-grained formulations to future work.}

Putting these pieces together, we arrive at the following population-level formulation. We seek a strategy that minimizes how often support is used, while guaranteeing that the agent rarely skips support on the instances where support would have been beneficial:

\begin{tcolorbox}[
    colback=gray!5,
    colframe=black!45,
    coltitle=black,
    colbacktitle=gray!5,
    title=\textbf{Strategic Decision Support Optimization},
    fonttitle=\normalsize\bfseries,
    boxrule=0.5mm,
    arc=2mm,
    top=1mm, bottom=1mm,
    left=1.5mm, right=1.5mm
]
\begin{equation}
\tag{SDS-Opt}\label{eq:sds_opt}
\begin{alignedat}{2}
    \phantom{\text{subject to}}\quad
    &\underset{a:\,\mathcal{X}\times\mathcal{Y}\to\{0,1\}}{\text{minimize}}
    &\quad& \mathbb{E}_{X}\bigl[\mathbf{1}\{a(X, Y_0)=1\}\bigr] \\[1ex]
    &\text{subject to}
    && \mathbb{P}\bigl(a(X, Y_0)=0 \mid g(X,Y_0,Y_1)=1\bigr) \;\leq\; \varepsilon
\end{alignedat}
\end{equation}
\end{tcolorbox}

The parameter $\varepsilon \in [0,1]$ specifies the tolerated level of missed-support error. Smaller values of $\varepsilon$ require the strategy to seek support on a larger fraction of instances where support would help, and therefore lead to more frequent support calls. Hence, this formulation captures the balance: support should be used sparingly, but not at the expense of missing the instances where it is genuinely needed.

We now characterize the optimal solution of \eqref{eq:sds_opt}. 




\begin{theorem}
\label{thm:optimal-population}
There exists an optimal solution to (SDS-Opt) of the form
\[
a^\star(x,y_0)=\mathbf{1}\{\operatorname{val}(x,y_0)>\tau^\star\},
\]
with possible randomization on the boundary , i.e. when 
$\{\operatorname{val}(x,y_0)=\tau^\star\}$, if needed.
\end{theorem}


Theorem~\ref{thm:optimal-population} identifies $\operatorname{val}(x,y_0)$ as the fundamental quantity for support seeking. The optimal strategy thresholds this quantity, i.e., support should be sought when its value is ``high enough.''

This characterization forms the foundation of our online method. Since $\operatorname{val}(x,y_0)$ is generally unknown in practice, the core task becomes to approximate it and threshold it adaptively. The algorithm in the next section does exactly this: it uses a score function as a proxy for the value of support, refines this score through calibration-on-the-fly, and updates a decision threshold online to determine when support should be sought.


\section{Online Algorithm and Guarantees} \label{sec:alg}

In this section, we design \emph{Strategic Oversight for Support-seeking} (SOS), an online algorithm for deciding when an AI agent should seek support. The population result of Section~\ref{sec:fundamentals} suggests a simple principle: estimate the value of support and seek support when this value is high enough. SOS turns this principle into a rigorous online procedure, where the value of support is estimated and calibrated online, and support decisions are made based on a threshold that is debiased sequentially.

We fix an AI agent and a support mechanism, and make no distributional assumptions on the data, the agent's behavior, or the behavior of the support mechanism. Building on the structural result of Section~\ref{sec:fundamentals}, SOS uses a score function
\[
s_\theta:\mathcal{X}\times\mathcal{Y}\to[0,1],
\]
where $s_\theta(x,y)$ evaluates an input $x$ together with a candidate output $y$. The score is intended to approximate the value of support, $\operatorname{val}(x,y)$, from Theorem~\ref{thm:optimal-population}. Here, $\theta$ denotes the parameters of the score function; these parameters may be fixed in advance, pretrained, or learned online as feedback is collected. Concrete choices and parameterizations of the score function are discussed in Section~\ref{sec:exp}.

\begin{tcolorbox}[colback=gray!5,colframe=black!80,boxrule=0.5pt,arc=2pt,left=4pt,right=4pt,top=4pt,bottom=4pt]
\textbf{Online interaction pipeline.}
At each round $t=1,2,\dots$:
\begin{enumerate}[leftmargin=5mm,itemsep=1pt,topsep=2pt]
    \item The agent receives an input $x_t$, produces an unsupported output $y_0^t$, and computes a score
    \[
    s_t := s_{\theta_t}(x_t, y_0^t).
    \]
    \item Based on $s_t$, the policy chooses an action $a_t\in\{0,1\}$, where $a_t=0$ means proceeding without seeking support, and $a_t=1$ means seeking support.
    \item If $a_t=0$, the agent finalizes the unsupported outcome $y_0^t$ and proceeds to the next round.
\item If $a_t=1$, the agent seeks support, produces the supported outcome $y_1^t$, and computes
\[
g_t := g(x_t,y_0^t,y_1^t)\in\{0,1\},
\]
which indicates whether support was beneficial on round $t$.
\end{enumerate}
\end{tcolorbox}

Our goal is to control the empirical missed-support error at a user-specified level $\alpha\in(0,1)$:
\footnote{We adopt the convention $\widehat{\mathrm{MSE}}(T)=0$ when $\sum_{t=1}^T g_t=0$.}
\begin{equation}
\label{eq:mse}
\widehat{\mathrm{MSE}}(T)
:=
\frac{\sum_{t=1}^T g_t(1-a_t)}{\sum_{t=1}^T g_t}
\le \alpha .
\end{equation}
This error is the fraction of beneficial-support rounds on which the policy nevertheless acts alone. Two remarks are in order.

\begin{itemize}[leftmargin=5mm,itemsep=2pt,topsep=2pt]
    \item We modeled the online interaction in a form where the support decision is made using both the input $x_t$ and the unsupported output $y_0^t$. In practice, one may choose to avoid generating, or only partially generate, the unsupported outcome before deciding whether to seek support. On rounds where support is sought, one may likewise be able to compute the value of $g$ without generating all of $y_0^t$. Our framework is easily tuned to such operational choices. We discuss these practical variants and their implications for cost and performance in more detail in Appendix~\ref{app:operational}.

    \item The error we seek to control depends on $g_t$, but $g_t$ is only revealed on rounds where support is sought, since computing it requires comparing $y_0^t$ and $y_1^t$. Thus, the relevant error is fundamentally counterfactual and only selectively observed. This partial-feedback structure is exactly what makes randomization necessary in the online algorithm.
\end{itemize}

\textbf{Algorithm.}
At each round $t$, we maintain a threshold $\lambda_t$ and define the support-seeking probability
\[
p_t :=
\begin{cases}
1, & s_t \ge \lambda_t,\\
\mu, & s_t < \lambda_t,
\end{cases}
\]
where $\mu\in(0,1)$ is a fixed exploration parameter. The action is then sampled as
$
a_t \sim \mathrm{Bernoulli}(p_t).
$

Thus, when the score exceeds the current threshold, the algorithm always seeks support; when the score falls below the threshold, it still seeks support with a small probability $\mu$ in order to obtain feedback about whether support would have been helpful.


On rounds where support is sought, we observe $g_t$, and the threshold is updated as
\[
\lambda_{t+1}
\leftarrow
\lambda_t
-
\eta_t\cdot
\frac{g_t a_t}{p_t}
\Big[(1-p_t)\mathbf{1}\{s_t<\lambda_t\}-\alpha\Big],
\]
where $\eta_t>0$ is a step size and $\alpha\in(0,1)$ is the target missed-support error level.

This update resembles online quantile-tracking: the quantity inside the brackets plays the role of an error-minus-target signal. In a deterministic threshold policy, this signal would be $\mathbf{1}\{s_t<\lambda_t\}-\alpha$, since rounds below the threshold are exactly those on which the agent proceeds without support. Here, however, the action is randomized: even when $s_t<\lambda_t$, support is still sought with probability $p_t=\mu$. Thus the realized missed-support event is not simply $\mathbf{1}\{s_t<\lambda_t\}$, but rather $(1-a_t)\mathbf{1}\{s_t<\lambda_t\}$. Since the only randomness comes from the algorithm’s internal exploration, we have
\[
\mathbb{E}\!\left[(1-a_t)\mathbf{1}\{s_t<\lambda_t\}\mid \mathcal{F}_{t-1},x_t,s_t\right]
=
(1-p_t)\mathbf{1}\{s_t<\lambda_t\},
\]
where the expectation is taken only over the Bernoulli draw $a_t\sim \mathrm{Bernoulli}(p_t)$. This explains the extra factor $(1-p_t)$ in the update. The prefactor $g_t a_t/p_t$ then serves as an importance-weighted correction for the fact that $g_t$ is observed only on rounds where support is sought.

\begin{algorithm}[t]
\caption{Strategic Oversight for Support-seeking (SOS)}
\label{alg:ads}
\footnotesize
\begin{algorithmic}
\REQUIRE Target $\alpha\in(0,1)$; exploration parameter $\mu\in(0,1)$; step sizes $\{\eta_t\}_{t\ge1}, \{\gamma_t\}_{t\ge1}$; initial threshold $\lambda_1$; initial score parameter $\theta_1$

\FOR{$t=1,2,\dots$}
    \STATE Receive input $x_t$ and generate $y_0^t$
    \STATE Compute score $s_t=s_{\theta_t}(x_t, y_0^t)\in[0,1]$

    \STATE \hrulefill\ \textit{Support-seeking probability and action}\ \hrulefill

    \STATE $p_t \gets \mu + (1-\mu)\,\mathbf{1}\{s_t \ge \lambda_t\}$
    \STATE Sample $a_t \sim \mathrm{Bernoulli}(p_t)$

    \STATE \hrulefill\ \textit{Decision and feedback}\ \hrulefill

    \IF{$a_t=0$}
        \STATE $\lambda_{t+1}\gets\lambda_t$, $\theta_{t+1}\gets\theta_t$
    \ELSE
        \STATE Seek support and compute $y_1^t$
        \STATE Compute $g_t = g(x_t,y_0^t,y_1^t)\in\{0,1\}$

        \STATE \hrulefill\ \textit{Threshold update and calibration-on-the-fly}\ \hrulefill

        \STATE $\lambda_{t+1} \gets \lambda_t
        - \eta_t \,
        \dfrac{g_t}{p_t}
        \Big((1-p_t)\mathbf{1}\{s_t<\lambda_t\}-\alpha\Big)$

        \STATE $\theta_{t+1} \gets \theta_t - \gamma_t \dfrac{a_t}{p_t}\nabla_\theta\!\left(s_\theta(x_t, y_0^t)-g_t\right)^2\Big|_{\theta=\theta_t}$
    \ENDIF
\ENDFOR
\end{algorithmic}
\end{algorithm}

\textbf{Calibration-on-the-fly:} Alongside the threshold update, we also allow the score itself to be updated online. Specifically, we treat the score as parameterized by $\theta_t$ and update
\[
\theta_{t+1}
\leftarrow
\theta_t
-
\gamma_t\frac{a_t}{p_t}\nabla_\theta\!\left(s_\theta(x_t)-g_t\right)^2\Big|_{\theta=\theta_t},
\]
where $\gamma_t>0$ is a step size. The factor $a_t/p_t$ again provides the appropriate importance-weighted correction, since $g_t$ is observed only on rounds where support is sought. At this level, we view this as an abstract online calibration procedure for improving how well the score tracks the latent value of support; concrete score parameterizations and design choices will be discussed in the next section.

The next result shows that the threshold update rule in Algorithm~1 yields a distribution-free finite-sample guarantee for controlling the empirical missed-support error.

\begin{theorem}\label{thm:finite-sample-control}
Run Algorithm~\ref{alg:ads} with constant step size $\eta_t\equiv\eta>0$,
exploration parameter $\mu\in(0,1)$, target level $\alpha\in(0,1-\mu)$,
and initial threshold $\lambda_1\in[0,1]$. Fix any horizon $T\ge 1$, and define
$
N_g(T):=\sum_{t=1}^T g_t .
$
Then for any $\delta\in(0,1)$, w.p. at least $1-\delta$ over the
algorithm's randomness,
\[
\widehat{\mathrm{MSE}}(T)
\le
\alpha+\Delta(N_g(T),\delta),
\]
where, for $N\ge 1$,
$
\Delta(N,\delta)
:=
\frac{1+2\eta/\mu}{\eta N}
+
\sqrt{\frac{8\log(4/\delta)}{\mu N}}
+
\frac{4\log(4/\delta)}{3\mu N},
$
and $\Delta(0,\delta):=0$.
\end{theorem}

Theorem~4.1 shows that the empirical missed-support error is controlled at the target level $\alpha$ up to a finite-sample term. The bound has two qualitatively different sources. The first term, of order $1/(\eta N)$, is the intrinsic error of online quantile tracking and would remain even if the value of support signal were always revealed\citep{gibbs2021adaptiveconformalinferencedistribution, angelopoulos2023conformalpidcontroltime, ramalingam2025relationship}. The remaining terms arise from the randomized exploration needed to obtain unbiased feedback. The dependence on $\mu$ makes the tradeoff explicit: larger exploration improves error control, but increases support usage. Finally, the validity guarantee comes from the threshold update rule and does not rely on calibration-on-the-fly, so it holds whether the score parameters are fixed or getting updated. The role of calibration-on-the-fly is instead to improve the score over time and reduce unnecessary support calls, as we demonstrate empirically.

\section{Experiments}
\label{sec:exp}

This section is organized as follows. Section~\ref{subsec:score-design} introduces three families of score designs. Section~\ref{subsec:tasks-baselines} describes the four tasks, datasets, and baseline methods. Section~\ref{subsec:error-control} empirically validates Theorem~\ref{thm:finite-sample-control}, showing that the algorithm controls the missed-support error at the target level $\alpha$ across all score families. Section~\ref{subsec:dss-efficiency} then fixes the missed-support error across methods and compares how often each invokes support. At matched error, the method calling support less often better identifies when support is needed, and the gap between methods reflects the value of richer score signals.
\subsection{Score Design}
\label{subsec:score-design}
 Algorithm~\ref{alg:ads} is agnostic to how the score $s_\theta(x)$ is constructed. Since the threshold update controls the missed-support error at level $\alpha$ regardless of the score's quality (Theorem~\ref{thm:finite-sample-control}), the score's role is to reduce how often support is called by separating instances where support would be beneficial from those on which it would not. We evaluate three score families, organized in increasing order of expressivity. 

\begin{itemize}[leftmargin=*, itemsep=4pt, topsep=2pt]
\item \textbf{Confidence score.} The simplest family takes the score directly from the black-box signal, with no learnable parameters or calibration-on-the-fly: $s(x) := \hat{g}_{\text{bb}}(x)$. The signal is thresholded as-is, and only $\lambda$ in Algorithm~\ref{alg:ads} is adapted online. This natural training-free baseline serves as a reference for parameterized scores, and reappears as the anchor in the anchored family below.

\item \textbf{Representation score.} The second family parameterizes the score as a linear function of a frozen embedding, $s_\theta(x) := \sigma(\theta^\top \phi(x))$. $\phi(x)\in\mathbb{R}^d$ embeds the input, $\sigma(\cdot)$ is the sigmoid that keeps the score in $[0,1]$, and only the linear coefficients $\theta\in\mathbb{R}^d$ are updated by calibration-on-the-fly. This is a linear probe over a representation that summarizes the input without committing to a particular notion of confidence; calibration-on-the-fly learns which directions in representation space predict whether support helps. The embedding $\phi$ can be any pretrained text representation. In the \emph{black-box} setting,  $\phi$ comes from a separate frozen encoder applied to the input. In the \emph{white-box} setting, we can also use the LLM's hidden states. The input to $\phi$ is itself a modeling choice; we use the prompt $x$ throughout and study alternatives (such as also including $y_0$ when generated) in Appendix~\ref{app:operational}.

\item \textbf{Anchored score.} The final family combines the parameterized linear term with the black-box signal in logit space: $s_\theta(x) := \sigma\bigl(\mathrm{logit}(\hat{g}_{\text{bb}}(x)) + \theta^\top \phi(x)\bigr)$. Here $\hat{g}_{\text{bb}}$ acts as an \emph{anchor} providing an initial estimate of the value of support, and $\theta^\top \phi(x)$ learns a residual correction in logit space. When the anchor is already well-aligned with $g$, the anchored score inherits its quality and calibration-on-the-fly need only learn small adjustments. When the anchor is uninformative or systematically biased, the linear term can override it. 
\end{itemize}

The Representation and Anchored families are each instantiated with three embedding choices, summarized below. Black-box (BB) variants apply a separate frozen encoder to the input; the white-box (WB) variant uses the LLM's own hidden state at the final input
token.

\begin{tcolorbox}[
    colback=gray!5,
    colframe=black!45,
    boxrule=0.5mm,
    arc=2mm,
    top=1mm, bottom=1mm,
    left=1.5mm, right=1.5mm
]
\renewcommand{\arraystretch}{1.15}
\setlength{\tabcolsep}{4pt}
\footnotesize
\begin{tabularx}{\linewidth}{@{}l c c X@{}}
\textbf{Variant} & \textbf{BB} & \textbf{WB} & \textbf{Embedding model} \\
\midrule
\textsc{Confidence}   & \checkmark & \checkmark & --- (no embedding; $s(x) = \hat{g}_{\text{bb}}(x)$) \\
\midrule
\textsc{Rep-MiniLM}, \textsc{Anc-MiniLM} & \checkmark &            & \texttt{all-MiniLM-L6-v2} sentence encoder \\
\textsc{Rep-Gemini}, \textsc{Anc-Gemini} & \checkmark &            & \texttt{gemini-embedding-2} text encoder \\
\textsc{Rep-Hidden}, \textsc{Anc-Hidden} &            & \checkmark & LLM's hidden state at the final input token \\
\end{tabularx}
\end{tcolorbox}

\subsection{Tasks and Baselines}
\label{subsec:tasks-baselines}
\paragraph{Tasks.} We instantiate each of the application categories from Section~\ref{sec:intro} on a concrete dataset. Task specifications and prompts are deferred to Appendix~\ref{app:prompts}.

\textbf{Information gathering} on \textbf{DDXPlus}~\citep{tchango2022ddxplusnewdatasetautomatic}(medical diagnosis): A diagnostic agent receives a patient's chief complaint and initial symptoms as $x$; support reveals follow-up questions, examination findings, and laboratory results. We set $g=1$ when $y_0$ is an incorrect diagnosis and $y_1$ is correct against the ground-truth pathology. 

\textbf{Tool use} on \textbf{WikiSQL}~\citep{zhong2017seq2sqlgeneratingstructuredqueries}: An agent receives a natural-language question about a table as $x$; support consists of formulating and executing SQL queries against it. We set $g=1$ when $y_0$ is incorrect and $y_1$ matches the dataset's expected answer. 

\textbf{Human-in-the-loop planning} on \textbf{VirtualHome}\citep{puig2018virtualhomesimulatinghouseholdactivities}: A household robot receives a task description as $x$ and  must output a plan as a sequence of actions; support reveals scene-specific object locations and constraints from the resident. $g=1$ when the longest-common-subsequence (LCS) overlap of $y_1$ with the gold action sequence exceeds that of $y_0$. 

\textbf{Human-AI collaborative reasoning} on Level 4--5 problems from \textbf{MATH} \citep{hendrycks2021measuringmathematicalproblemsolving}: An reasoning agent attempts solving math problems; support comes from a stronger reasoner that provides targeted guidance on a step the agent identifies as uncertain, without revealing the full solution. We set $g=1$ when $y_0$ is incorrect and $y_1$ is correct.

We run all four tasks with three frontier base agents spanning both regimes: Qwen-2.5-7B~\cite{qwen2025qwen25technicalreport} (white-box), Gemini-2.5-Flash~\cite{comanici2025gemini25pushingfrontier} (black-box), and GPT-4o-mini~\cite{openai2024gpt4ocard} (black-box).

\textbf{Baselines. } Our primary baseline is \textbf{LLM-decides}, which lets the LLM itself choose whether to seek support after producing $y_0$. This is a natural reference point as it isolates what the agent does on its own, with no oversight layer, no error target, and no calibration. Its missed-support rate is a property of the model and the task, and cannot be controlled. We take this rate as the target $\alpha$ for our algorithm on each (task, model) pair, so that all comparisons are on equal footing. We then run each score variant from Section~\ref{subsec:score-design} inside Algorithm~\ref{alg:ads}: \textsc{CONFIDENCE} adapts only $\lambda_t$, while the Representation and Anchored variants additionally update the score parameters $\theta$ via calibration-on-the-fly. 

\subsection{Error Control}
\label{subsec:error-control}

\begin{figure}[htp]
    \centering
    \includegraphics[width=\linewidth]{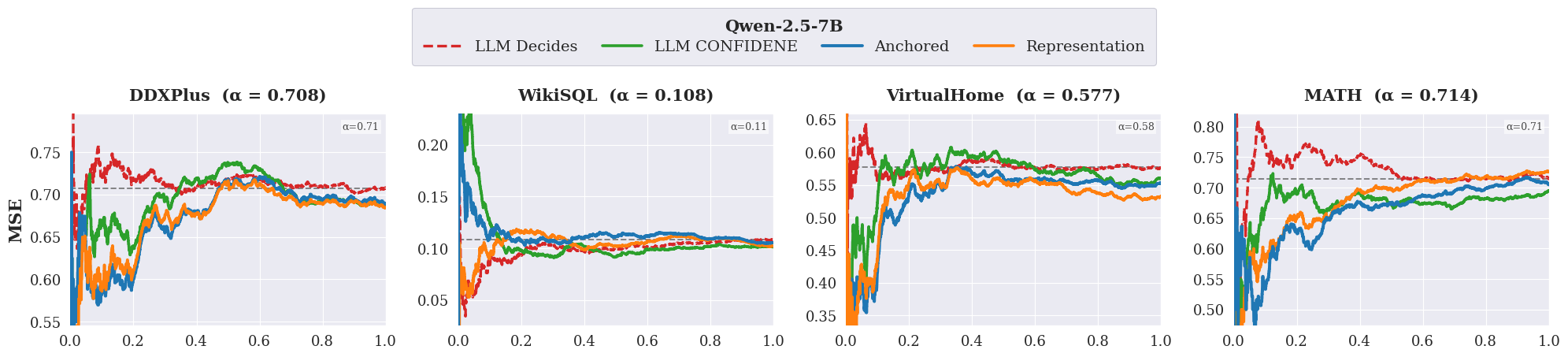}
    \caption{Cumulative missed-support error on all four tasks with Qwen-2.5-7B as the agent.}
    \label{fig:error-control}
\end{figure}
We track the cumulative missed-support error $\mathrm{MSE}(T)$ (Equation~\eqref{eq:mse}), the empirical fraction of rounds where support would have helped but was not invoked.
Figure~\ref{fig:error-control} shows the running $\mathrm{MSE}(T)$ for one representative model across all four tasks with the $x$-axis indicating progress through the interaction stream as a fraction of total rounds. All score variants converge to the target $\alpha$ (the rate achieved by LLM-Decides), with trajectories following the typical adaptation profile of online quantile tracking. Convergence rate is governed by the threshold step size $\eta$, which trades off adaptation speed against stability. The same behavior hold across all base agents. The full set is reported in Appendix~\ref{app:error-control-all}.



\subsection{Support Efficiency}
\label{subsec:dss-efficiency}

Error control alone does not imply an efficient support policy. A policy that always invokes support has zero missed-support error but is maximally wasteful. We therefore compare methods on equal footing by asking how often each invokes support at the same missed-support error level. Figure~\ref{fig:dss-main} reports the cumulative support rate across all task--model pairs, with the target $\alpha$ for SOS set to the missed-support rate achieved by \textsc{LLM-Decides} on each pair. Two general patterns emerge.


\begin{figure}[htp]
    \centering
    \includegraphics[width=\linewidth]{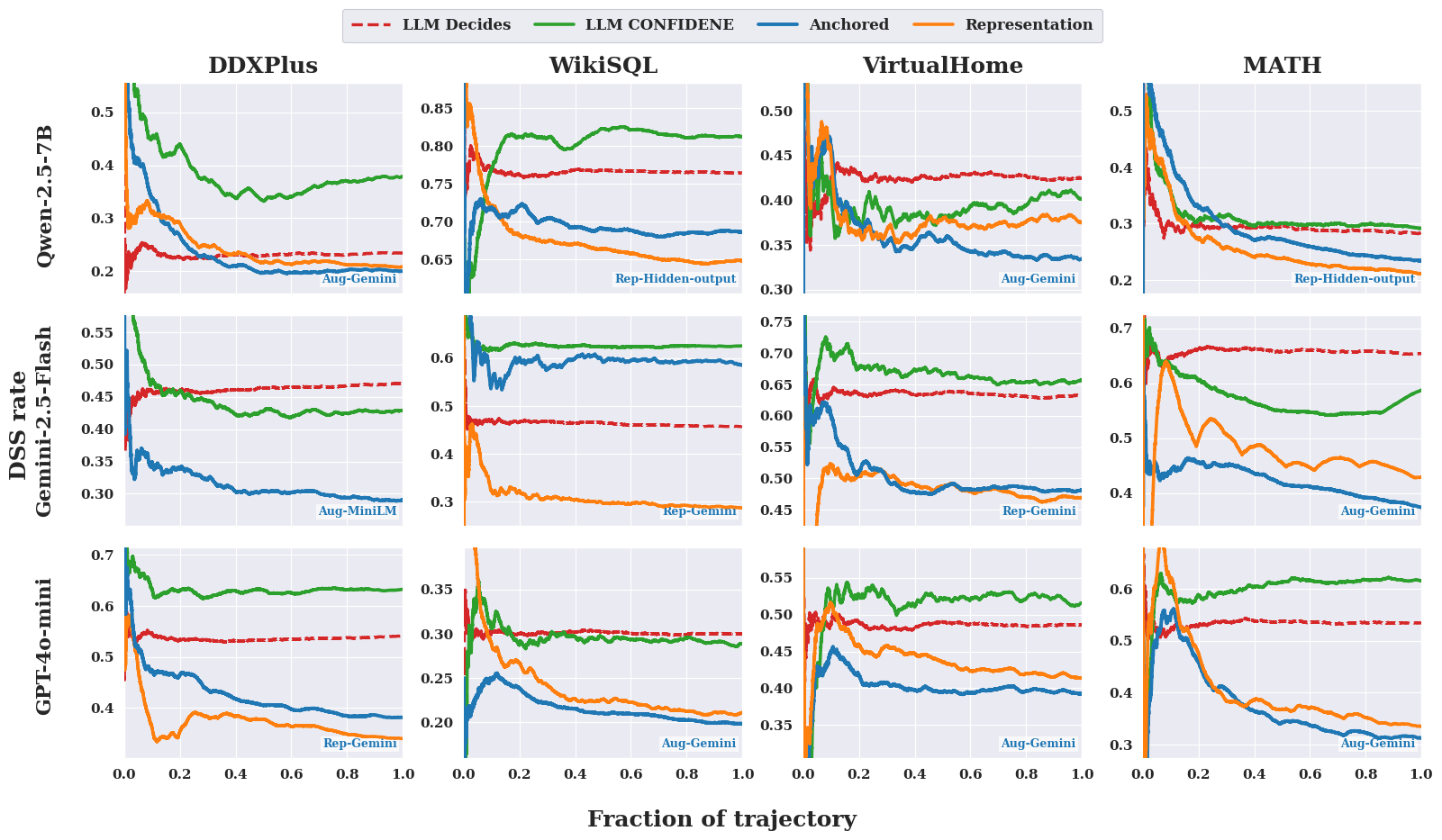}
    \caption{Cumulative support rate $\widehat{\mathrm{SR}}_T = \frac{1}{T}\sum_{t=1}^{T} a_t$ across all task--model pairs at matched missed-support error. We show the best-performing variant across both families, paired with its same-embedding counterpart from the other family. Full per-panel comparisons in Appendix~\ref{app:dss-all-variants}.}
    \label{fig:dss-main}
\end{figure}
\textbf{Calibration-on-the-fly recovers from uninformative signals.} The Representation family reliably reduces the support rate relative to \textsc{LLM-decides} on all task--model pairs, and the Anchored family does so on 11 of 12, regardless of how informative the underlying signal is. The oversight layer does not require a strong starting signal, since calibration-on-the-fly improves it from feedback on rounds where support is invoked. The Confidence score is unstable across tasks and agents; sometimes below \textsc{LLM-decides} (e.g., VirtualHome on Qwen-2.5-7B), sometimes well above (e.g., WikiSQL on Qwen-2.5-7B). The parameterized families absorb this variability and reach lower support rates.

\textbf{Anchoring on a reliable initial signal accelerates the gain.} The two parameterized families differ in whether they treat the initial signal as a useful prior or discard it. When the anchor is informative, calibration-on-the-fly only needs to learn a small correction, and the Anchored variants outperform their Representation counterparts (e.g., MATH on Gemini-2.5-Flash). When the anchor is uninformative or actively misleading, the Representation family pulls ahead, since the residual term in the Anchored score has to first undo the bad anchor before it can learn (e.g., WikiSQL on Gemini-2.5-Flash). In our experiments the anchor is the LLM's own self-confidence, but the same principle applies to any auxiliary signal. 

\begin{framed}
\noindent \textbf{Takeaway.} AI agents struggle to determine when seeking support is beneficial, even when they are proficient at utilizing that support once provided. Our strategic oversight algorithm significantly outperforms an agent’s self-decisions, achieving lower support rates while maintaining formal error control. Among score designs, parameterized scores that learn from representations perform reliably well, with calibration-on-the-fly effectively reducing the support rate as the query stream progresses.



\end{framed}


\section{Limitations and Future Work}
We focus on a binary support-seeking decision and measure cost by the frequency of support calls. In practice, agents may have multiple support options, and the cost of support may vary across options and instances. One may also replace the binary value indicator $g$ with a continuous measure of how much value support adds, defining the value of support as its expectation rather than the probability that $g=1$. We defer these finer-grained formulations to future work.

\section{Extended Related Works}
\label{app:extended-related-works}

\paragraph{Training Agents to Seek Support.}
A significant body of work bakes support-seeking behavior into the agent
itself through training, modifying the agent's weights or generation
procedure so that it learns when to invoke external resources or ask
clarifying questions. On the tool-use and retrieval side, Toolformer~\citep{schick2023toolformerlanguagemodelsteach}
self-supervisedly teaches a language model where to insert API calls by
sampling and filtering candidate calls based on whether they reduce
future-token loss. Self-RAG~\citep{asai2023selfraglearningretrievegenerate}
trains a model to interleave retrieval with generation through reflection
tokens that signal when to retrieve and when to commit. Confidence-triggered
retrieval methods take a complementary route, deciding to retrieve based on
signals during decoding: FLARE~\citep{jiang2023activeretrievalaugmentedgeneration}
triggers retrieval when the predicted next sentence contains low-confidence
tokens, Adaptive-RAG~\citep{jeong2024adaptiveraglearningadaptretrievalaugmented}
trains a small classifier to route queries among no-retrieval, single-step,
and iterative-retrieval strategies, UALA~\citep{han2024uncertaintyawarelanguageagent}
sets an uncertainty threshold on agent answers and resorts to external
resources when exceeded, and SCENT~\citep{qian2026scent} formalizes
adaptive retrieval as reinforcement learning with dense rewards on
intermediate retrieval quality. On the user-interaction side, GATE~\citep{li2023elicitinghumanpreferenceslanguage}
introduced the framework of having an LLM elicit user preferences through
targeted questions, with subsequent work training the questioner via
self-improvement (STaR-GATE~\citep{andukuri2024stargateteachinglanguagemodels}),
DPO-style preference optimization at the conversational-action
level~\citep{zhang2025modelingfutureconversationturns}, and benchmarks that
explicitly evaluate the ask-or-tool-call decision~\citep{ross2025when2callnottools}.

We instead provide an oversight layer around an unmodified agent, which
decides whether to invoke support without changing how the agent itself
generates outputs. This separation makes the two lines complementary
rather than competing: any model that has already been trained to call
tools, retrieve, or ask clarifying questions is precisely the kind of
base policy our framework wraps. This enables us to deliver finite-sample
error control on top of strong, already-trained agents without
retraining, fine-tuning, or modifying generation, and to do so uniformly
across support modalities. 
Our experiments use state-of-the-art LLMs (Qwen-2.5-7B,
Gemini-2.5-Flash, GPT-4o-mini) that already incorporate substantial
training of this kind, and our gains are reported on top of these strong
agents.

\paragraph{Allocating Tasks Across Decision-Making Systems.}
A second body of work studies how to allocate tasks across multiple
decision-making systems rather than how to support a single one. Within
this body, four directions are most relevant. LLM cascades and routers
send each query to the cheapest system whose output is good enough:
FrugalGPT~\citep{chen2023frugalgptuselargelanguage} composes LLM APIs
into a cascade with learned stop criteria, RouteLLM~\citep{ong2025routellmlearningroutellms}
trains a router that selects between strong and weak LLMs using human
preference data, and cost-saving cascades with early
abstention~\citep{zellinger2025costsavingllmcascadesearly} extend cascades
to allow abstention at intermediate tiers. 

Learning to defer~\citep{madras2018predictresponsiblyimprovingfairness}
generalizes rejection learning by routing uncertain instances to an
expert whose decision replaces the model's, with subsequent work
developing Bayes-consistent surrogate
losses~\citep{mozannar2021consistentestimatorslearningdefer},
calibrated one-vs-all
approaches~\citep{verma2022calibratedlearningdeferonevsall},
extensions to multiple
experts~\citep{verma2023learningdefermultipleexperts,NEURIPS2023_0b17d256,hemmer2023learningdeferlimitedexpert,keswani2021unbiasedaccuratedeferralmultiple,wilder-2021-learing-to-complement-humans},
LLM allocation~\citep{montreuil2025optimalqueryallocationextractive},
algorithmic-triage formulations~\citep{raghu2019algorithmicautomationproblemprediction,okati2021differentiablelearningtriage,de2021classificationhumanassistance,de2021regressionhumanassistance},
training-free conformal deferral~\citep{charusaie2022sampleefficientlearningpredictors},
and collaborative matching for selective
deferral~\citep{arnaizrodriguez2025humanaicomplementaritymatchingtasks}.
Selective prediction lets the model abstain entirely on uncertain inputs, trading coverage for accuracy~\citep{JMLR:v11:el-yaniv10a,geifman2017selectiveclassificationdeepneural,geifman2019selectivenetdeepneuralnetwork}.
A more recent line uses conformal calibration to deliver formal guarantees
on these decisions: Conformal Arbitrage~\citep{overman2025conformalarbitrageriskcontrolledbalancing}
calibrates a threshold mediating between a primary model and a more
conservative guardian, and Calibrate-Then-Delegate (CTD)~\citep{pona2026calibratethendelegatesafetymonitoringrisk}
introduces a delegation-value probe that predicts the benefit of
escalating to an expert and calibrates a threshold on this probe via
offline multiple-hypothesis testing.

We instead study when an unmodified agent should invoke support to
augment its own reasoning, rather than which of two systems should
produce the output. Although we study a different problem, the combination of distribution-free online
calibration, importance-weighted threshold updates under counterfactual
partial feedback, and control of a conditional miss-rate, together with the 
algorithmic techniques we develop may be of
independent interest to the broader umbrella L2D and selective classification literatures.

\paragraph{Inference-Time Policies for Support-Seeking.}
A closely related line of work designs inference-time policies that
decide whether an agent should seek information before acting.
CLAM~\citep{kuhn2023clamselectiveclarificationambiguous} prompts an LLM
to classify whether a user question is ambiguous and asks a clarifying
question when it is. AwN~\citep{wang2026learningaskllmagents} prompts
the LLM to ask whenever it encounters obstacles caused by unclear
instructions during tool use. Earlier proactive-dialogue work uses
prompting to elicit clarifying behavior more
generally~\citep{deng2023promptingevaluatinglargelanguage}. More recent work introduces formal objectives:
KnowNo~\citep{ren2023robotsaskhelpuncertainty} uses split conformal
prediction on a held-out calibration set to construct prediction sets
over candidate user intents, asking for help whenever the set is
non-singleton; the Value-of-Information policy
of~\citet{dong2026valueinformationframeworkhumanagent} maintains a
belief distribution over latent intents, generates candidate clarifying
questions, simulates user replies, and asks when the expected utility
gain exceeds an explicit asking cost.  A complementary perspective is taken by \citet{kiyani2026trustcheapcheckweak},
who frame support-seeking on the \emph{verification} side rather than
the generation side: their weak--strong verification policies decide
when a cheap internal check (e.g., self-consistency, a proxy reward)
suffices and when to defer to costly external verification (e.g., user
inspection), with an online algorithm that provably controls both
acceptance and rejection errors.

We instead provide a unifying framework that brings these approaches
under the same design principles, through an oversight layer with
rigorous finite-sample error control. This enables us to handle a broad
set of support modalities --- clarifying questions, external tools,
additional evidence, expert guidance --- within a single algorithm, to
operate fully online without an offline calibration set, and to control
a counterfactual conditional error at a user-chosen level. Our
experiments target precisely these capabilities: covering information
gathering, tool use, and human-AI collaboration under the same algorithm,
with online error control across a range of target levels.

\paragraph{AI as Decision Support System for humans}
A long line of work studies how AI systems should support human
decision-makers. One strand uses prediction sets as a structured
interface for collaboration, calibrating the set to improve human
accuracy~\citep{straitouri2023improvingexpertpredictionsconformal,babbar2022utilitypredictionsetshumanai,benz2025humanalignmentinfluencesutilityaiassisted,detoni2024humanaicomplementaritypredictionsets,Wang2022ImprovingSP},
analyzing decision-relevant uncertainty~\citep{hullman2025conformalpredictionhumandecision},
and formalizing counterfactual harm from set-based
support~\citep{straitouri2024controllingcounterfactualharmdecision,noorani2025human, noorani2026multi}.
A second strand develops theoretical frameworks for human-AI
complementarity, identifying when joint systems surpass either agent
alone~\citep{bansal2021accurateaibestteammate,steyvers,rastogi2023taxonomyhumanmlstrengths,vaccaro2024combinations}
and how algorithmic outputs shape human
choices~\citep{cowgill-stevenson,donahue2024listsbetteronebenefits,Kleinberg_2021,chan2019assistivemultiarmedbandit,bordt-and-vonluxburg,Tian2023,gao2025humancenteredhumanaicollaborationhchac}. recent surveys taxonomize the broader space of
human-machine hybrid decision-making
paradigms~\citep{punzi}.

In all of these, the human is the final decision-maker and the AI
provides support around them. Our work reverses this configuration: the
AI agent is the decision-maker, and humans, tools, and auxiliary
information serve as support mechanisms around it.

\paragraph{Uncertainty quantification for AI Agents}
A separate literature studies how
to elicit reliable uncertainty estimates from LLMs more generally.
Conformal prediction methods construct prediction sets with formal
coverage guarantees over open-ended LLM
outputs~\citep{quach2024conformallanguagemodeling,shahrokhi2025conformal,kladny2025conformalgenerativemodelingimproved,su2024apienoughconformalprediction,ulmer2024nonexchangeableconformallanguagegeneration,cherian2024largelanguagemodelvalidity,mohri2024languagemodelsconformalfactuality,rubinconformal,liu2024multigroupuncertaintyquantificationlongform,kumar2023conformalpredictionlargelanguage};
a complementary line uses conformal risk control to learn abstention
or selective-prediction
policies~\citep{yadkori2024mitigatingllmhallucinationsconformal,tayebati2025learningconformalabstentionpolicies,yadkori2024believebelievellm}.
A broader thread mitigates hallucinations through direct uncertainty
estimation, semantic-disagreement signals across sampled responses, or
verifier-based
detectors~\citep{liu2024uncertaintyestimationquantificationllms,aichberger2024semanticallydiverselanguagegeneration,farquhar2024detecting,duan2024shiftingattentionrelevancepredictive,wang2023selfconsistencyimproveschainthought,kuhn2023semanticuncertaintylinguisticinvariances,manakul2023selfcheckgptzeroresourceblackboxhallucination},
and recent work fine-tunes LLMs for better-calibrated verbalized
confidence~\citep{kapoor2025largelanguagemodelstaught,li2025conftunertraininglargelanguage}.
These methods produce signals about the agent's confidence in its own
output. Better LLM uncertainty estimators from this line could be used as the anchor in our anchored score, providing a stronger initial signal that calibration-on-the-fly need only refine.

\section{Acknowledgments}
The authors thank EnCORE, the Institute for Emerging CORE Methods in Data Science, for their support, as well as NSF award 2502489 under the MFAI: Mathematical Foundations of Alignment in Generative Artificial Intelligence program. SK additionally acknowledges support from a gift from AWS to Penn Engineering’s ASSET Center for Trustworthy AI.

\bibliography{references} 
\bibliographystyle{plainnat} 

\appendix
\clearpage
\appendix
\tableofcontents

\newpage

\section{Proofs}
\subsection{Proof of Theorem \ref{thm:optimal-population}}
\begin{proof}
Let
\[
G := g(X,Y_0,Y_1), \qquad \pi_1 := \mathbb{P}(G=1),
\]
and assume $\pi_1>0$. For any measurable strategy
$a:\mathcal{X}\times\mathcal{Y}\to\{0,1\}$, the constraint
\[
\mathbb{P}(a(X,Y_0)=0\mid G=1)\le \varepsilon
\]
is equivalent to
\[
\mathbb{P}(a(X,Y_0)=1\mid G=1)\ge 1-\varepsilon.
\]
By the definition
\[
\operatorname{val}(x,y_0):=\mathbb{P}(G=1\mid X=x,Y_0=y_0),
\]
we have
\[
\mathbb{P}(a(X,Y_0)=1\mid G=1)
=
\frac{\mathbb{E}\!\left[a(X,Y_0)\mathbf{1}\{G=1\}\right]}{\pi_1}
=
\frac{\mathbb{E}\!\left[a(X,Y_0)\operatorname{val}(X,Y_0)\right]}{\pi_1},
\]
where the second equality follows from the tower property. Hence
(SDS-Opt) is equivalent to
\[
\min_{a:\mathcal{X}\times\mathcal{Y}\to\{0,1\}}
\mathbb{E}[a(X,Y_0)]
\qquad
\text{subject to}
\qquad
\mathbb{E}[a(X,Y_0)\operatorname{val}(X,Y_0)]\ge (1-\varepsilon)\pi_1.
\]

Let
\[
\beta := (1-\varepsilon)\pi_1.
\]
The objective charges one unit for every point on which support is sought, while the constraint
credits such a point in proportion to $\operatorname{val}(X,Y_0)$. Therefore, an optimal strategy should allocate
support to the largest values of $\operatorname{val}$ first. We now make this formal.

Choose a threshold $\tau^\star\in[0,1]$ such that
\[
\mathbb{E}\!\left[\operatorname{val}(X,Y_0)\mathbf{1}\{\operatorname{val}(X,Y_0)>\tau^\star\}\right]
\le \beta
\le
\mathbb{E}\!\left[\operatorname{val}(X,Y_0)\mathbf{1}\{\operatorname{val}(X,Y_0)\ge \tau^\star\}\right].
\]
If necessary, randomize on the boundary $\{\operatorname{val}(X,Y_0)=\tau^\star\}$ with probability
$\rho^\star\in[0,1]$ chosen so that
\[
\mathbb{E}\!\left[
\operatorname{val}(X,Y_0)
\Big(
\mathbf{1}\{\operatorname{val}(X,Y_0)>\tau^\star\}
+
\rho^\star \mathbf{1}\{\operatorname{val}(X,Y_0)=\tau^\star\}
\Big)
\right]
=
\beta.
\]
Define
\[
a^\star(x,y_0)
=
\mathbf{1}\{\operatorname{val}(x,y_0)>\tau^\star\}
+
\rho^\star \mathbf{1}\{\operatorname{val}(x,y_0)=\tau^\star\}.
\]
Equivalently, $a^\star$ seeks support whenever $\operatorname{val}(x,y_0)$ is above the threshold and randomizes
on the boundary if needed. By construction, $a^\star$ satisfies the constraint.

It remains to show optimality. Let $a$ be any feasible strategy, allowing randomized values in
$[0,1]$; this only enlarges the feasible class. By the definition of $a^\star$, we have pointwise
\[
(a(x,y_0)-a^\star(x,y_0))(\operatorname{val}(x,y_0)-\tau^\star)\le 0.
\]
Taking expectations gives
\[
\mathbb{E}\!\left[(a-a^\star)\operatorname{val}\right]
\le
\tau^\star \mathbb{E}[a-a^\star].
\]
Since $a$ is feasible and $a^\star$ satisfies the constraint with equality,
\[
\mathbb{E}\!\left[(a-a^\star)\operatorname{val}\right]\ge 0.
\]
Therefore,
\[
0 \le \tau^\star \mathbb{E}[a-a^\star].
\]
If $\tau^\star>0$, this implies
\[
\mathbb{E}[a]\ge \mathbb{E}[a^\star].
\]
If $\tau^\star=0$, feasibility requires $a=1$ almost surely on the set where
$\operatorname{val}(X,Y_0)>0$, and $a^\star$ is the minimal such rule, so the same conclusion holds.
Thus no feasible strategy uses support less often than $a^\star$.

Hence an optimal solution is given by a threshold rule in $\operatorname{val}(x,y_0)$, with possible randomization
on the boundary. This proves the claim.
\end{proof}
\subsection{Proof of Theorem \ref{thm:finite-sample-control}}

\begin{proof}
Fix a horizon $T$. For the analysis, define
\[
g_t:=g(x_t,y_0^t,y_1^t)\in\{0,1\}
\]
as the potential benefit of support on round $t$, whether or not support is actually sought. All
probabilities below are over the algorithm's internal Bernoulli draws. If
\[
N_g(T):=\sum_{t=1}^T g_t=0,
\]
then $\widehat{\mathrm{MSE}}(T)=0$ by convention, and the claim is immediate. Hence assume
\[
N:=N_g(T)=\sum_{t=1}^T g_t\ge 1.
\]

We first outline the proof. The threshold update controls a threshold-induced version of the
missed-support error, in which the realized missed-support indicator $(1-a_t)$ is replaced by its
conditional mean $(1-p_t)\mathbf{1}\{s_t<\lambda_t\}$. The update gives a telescoping identity for
this quantity, up to an importance-weighted martingale error. We control this martingale error by
Freedman's inequality. Finally, we compare the threshold-induced error to the actual empirical error
$\widehat{\mathrm{MSE}}(T)$ using a second martingale concentration bound.

\medskip

\noindent\emph{Claim 1.} Define
\[
\overline{\mathrm{MSE}}(T)
:=
\frac{1}{N}\sum_{t=1}^T
g_t(1-p_t)\mathbf{1}\{s_t<\lambda_t\}.
\]
Then, with probability at least $1-\delta/2$,
\[
\overline{\mathrm{MSE}}(T)
\le
\alpha
+
\frac{1+2\eta/\mu}{\eta N}
+
\sqrt{\frac{2\log(4/\delta)}{\mu N}}
+
\frac{\log(4/\delta)}{3\mu N}.
\]

\noindent\emph{Proof of Claim 1.}
Define
\[
\bar e_t
:=
g_t\Big((1-p_t)\mathbf{1}\{s_t<\lambda_t\}-\alpha\Big),
\qquad
\hat e_t
:=
\frac{g_ta_t}{p_t}
\Big((1-p_t)\mathbf{1}\{s_t<\lambda_t\}-\alpha\Big).
\]
The threshold update can be written as
\[
\lambda_{t+1}=\lambda_t-\eta \hat e_t.
\]
Therefore,
\[
\sum_{t=1}^T \hat e_t
=
\frac{\lambda_1-\lambda_{T+1}}{\eta}.
\]
Moreover,
\[
\sum_{t=1}^T \bar e_t
=
N\big(\overline{\mathrm{MSE}}(T)-\alpha\big).
\]

Let $\mathcal{F}_{t-1}$ denote the history before round $t$, and define
\[
\mathcal{G}_t
:=
\sigma(\mathcal{F}_{t-1},x_t,y_0^t,s_t,\lambda_t,p_t,g_t).
\]
This filtration contains all pre-action quantities at round $t$, together with the potential benefit
variable $g_t$ used only for the analysis. Conditional on $\mathcal{G}_t$, the only randomness in
round $t$ is
\[
a_t\sim \mathrm{Bernoulli}(p_t).
\]
Hence
\[
\mathbb{E}\!\left[\hat e_t\mid \mathcal{G}_t\right]
=
g_t\Big((1-p_t)\mathbf{1}\{s_t<\lambda_t\}-\alpha\Big)
\mathbb{E}\!\left[\frac{a_t}{p_t}\,\middle|\,\mathcal{G}_t\right]
=
\bar e_t.
\]
Thus, with
\[
Z_t:=\hat e_t-\bar e_t,
\]
the sequence $(Z_t)_{t\ge 1}$ is a martingale difference sequence, and
\[
N\big(\overline{\mathrm{MSE}}(T)-\alpha\big)
=
\frac{\lambda_1-\lambda_{T+1}}{\eta}
-
\sum_{t=1}^T Z_t.
\]

We next bound the threshold sequence. Since $p_t\ge \mu$ and
\[
\Big|(1-p_t)\mathbf{1}\{s_t<\lambda_t\}-\alpha\Big|\le 1,
\]
we have
\[
|\hat e_t|\le \frac{1}{\mu}.
\]
If $\lambda_t<0$, then $s_t\ge \lambda_t$ because $s_t\in[0,1]$, so $p_t=1$ and
$\mathbf{1}\{s_t<\lambda_t\}=0$. Hence
\[
\hat e_t=-\alpha g_ta_t\le 0,
\]
and therefore $\lambda_{t+1}\ge \lambda_t$. If $\lambda_t>1$, then $s_t<\lambda_t$, so
$p_t=\mu$ and $\mathbf{1}\{s_t<\lambda_t\}=1$. Since $\alpha<1-\mu$,
\[
\hat e_t
=
\frac{g_ta_t}{\mu}\big((1-\mu)-\alpha\big)
\ge 0,
\]
and therefore $\lambda_{t+1}\le \lambda_t$. Starting from $\lambda_1\in[0,1]$, these facts imply
by induction that
\[
\lambda_t\in\left[-\frac{\eta}{\mu},\,1+\frac{\eta}{\mu}\right]
\qquad\text{for all }t\le T+1.
\]
Consequently,
\[
|\lambda_{T+1}-\lambda_1|
\le
1+\frac{2\eta}{\mu}.
\]

It remains to control the martingale term. Write
\[
Z_t
=
c_t\left(\frac{a_t}{p_t}-1\right),
\qquad
c_t
:=
g_t\Big((1-p_t)\mathbf{1}\{s_t<\lambda_t\}-\alpha\Big).
\]
Since $|c_t|\le 1$ and $p_t\ge\mu$,
\[
|Z_t|\le \frac{1}{\mu}.
\]
Furthermore,
\[
\operatorname{Var}(Z_t\mid\mathcal{G}_t)
=
c_t^2
\operatorname{Var}\!\left(\frac{a_t}{p_t}\,\middle|\,\mathcal{G}_t\right).
\]
Because $a_t\sim\mathrm{Bernoulli}(p_t)$,
\[
\operatorname{Var}\!\left(\frac{a_t}{p_t}\,\middle|\,\mathcal{G}_t\right)
=
\frac{1-p_t}{p_t}
\le
\frac{1}{\mu}.
\]
Also $c_t^2\le g_t$, since $g_t\in\{0,1\}$. Therefore,
\[
\operatorname{Var}(Z_t\mid\mathcal{G}_t)
\le
\frac{g_t}{\mu}.
\]
Thus the predictable quadratic variation satisfies
\[
\sum_{t=1}^T \operatorname{Var}(Z_t\mid\mathcal{G}_t)
\le
\frac{N}{\mu}.
\]
By Freedman's inequality, with probability at least $1-\delta/2$,
\[
\left|\sum_{t=1}^T Z_t\right|
\le
\sqrt{\frac{2N\log(4/\delta)}{\mu}}
+
\frac{\log(4/\delta)}{3\mu}.
\]
Combining the preceding displays gives
\[
\overline{\mathrm{MSE}}(T)
\le
\alpha
+
\frac{1+2\eta/\mu}{\eta N}
+
\sqrt{\frac{2\log(4/\delta)}{\mu N}}
+
\frac{\log(4/\delta)}{3\mu N}.
\]
This proves Claim 1. \hfill$\diamond$

\medskip

\noindent\emph{Claim 2.} With probability at least $1-\delta/2$,
\[
\widehat{\mathrm{MSE}}(T)
\le
\overline{\mathrm{MSE}}(T)
+
\sqrt{\frac{2\log(4/\delta)}{N}}
+
\frac{\log(4/\delta)}{3N}.
\]

\noindent\emph{Proof of Claim 2.}
By definition,
\[
\widehat{\mathrm{MSE}}(T)
=
\frac{1}{N}\sum_{t=1}^T g_t(1-a_t).
\]
Since $p_t=1$ whenever $s_t\ge\lambda_t$, we have $a_t=1$ almost surely on those rounds.
Therefore,
\[
\widehat{\mathrm{MSE}}(T)
=
\frac{1}{N}\sum_{t=1}^T
g_t(1-a_t)\mathbf{1}\{s_t<\lambda_t\}.
\]
Define
\[
U_t
:=
g_t\mathbf{1}\{s_t<\lambda_t\}
\Big((1-a_t)-(1-p_t)\Big)
=
g_t\mathbf{1}\{s_t<\lambda_t\}(p_t-a_t).
\]
Then
\[
\widehat{\mathrm{MSE}}(T)-\overline{\mathrm{MSE}}(T)
=
\frac{1}{N}\sum_{t=1}^T U_t.
\]
Moreover,
\[
\mathbb{E}[U_t\mid\mathcal{G}_t]=0,
\]
so $(U_t)_{t\ge 1}$ is a martingale difference sequence. Since $|U_t|\le 1$ and
\[
\operatorname{Var}(U_t\mid\mathcal{G}_t)
=
g_t\mathbf{1}\{s_t<\lambda_t\}
\operatorname{Var}(a_t\mid\mathcal{G}_t)
=
g_t\mathbf{1}\{s_t<\lambda_t\}p_t(1-p_t)
\le
g_t,
\]
the predictable quadratic variation is at most $N$. Freedman's inequality gives, with probability at
least $1-\delta/2$,
\[
\left|\sum_{t=1}^T U_t\right|
\le
\sqrt{2N\log(4/\delta)}
+
\frac{\log(4/\delta)}{3}.
\]
Dividing by $N$ proves Claim 2. \hfill$\diamond$

\medskip

Combining Claims 1 and 2 with a union bound, with probability at least $1-\delta$,
\[
\widehat{\mathrm{MSE}}(T)
\le
\alpha
+
\frac{1+2\eta/\mu}{\eta N}
+
\sqrt{\frac{2\log(4/\delta)}{\mu N}}
+
\frac{\log(4/\delta)}{3\mu N}
+
\sqrt{\frac{2\log(4/\delta)}{N}}
+
\frac{\log(4/\delta)}{3N}.
\]
Since $\mu\le 1$,
\[
\sqrt{\frac{2\log(4/\delta)}{N}}
\le
\sqrt{\frac{2\log(4/\delta)}{\mu N}},
\qquad
\frac{\log(4/\delta)}{3N}
\le
\frac{\log(4/\delta)}{3\mu N}.
\]
Therefore,
\[
\widehat{\mathrm{MSE}}(T)
\le
\alpha
+
\frac{1+2\eta/\mu}{\eta N}
+
2\sqrt{\frac{2\log(4/\delta)}{\mu N}}
+
\frac{2\log(4/\delta)}{3\mu N}.
\]
This is slightly stronger than the stated bound. Since
\[
2\sqrt{\frac{2\log(4/\delta)}{\mu N}}
=
\sqrt{\frac{8\log(4/\delta)}{\mu N}},
\]
and
\[
\frac{2\log(4/\delta)}{3\mu N}
\le
\frac{4\log(4/\delta)}{3\mu N},
\]
we obtain
\[
\widehat{\mathrm{MSE}}(T)
\le
\alpha
+
\frac{1+2\eta/\mu}{\eta N}
+
\sqrt{\frac{8\log(4/\delta)}{\mu N}}
+
\frac{4\log(4/\delta)}{3\mu N}.
\]
Recalling that $N=N_g(T)$ proves the theorem.
\end{proof}
\section{Additional Experimental Results}
\subsection{Support rate across all score variants}
\label{app:dss-all-variants}

Figure~\ref{fig:error-control} in the main text shows the cumulative support rate for the best-performing variant in each of the Anchored and Representation families, together with its counterpart in the other family. Figure~\ref{app-fig:dss-all-variants} reports the same quantity for every score variant we consider, providing the full picture of how each embedding choice and family performs at matched missed-support error. 
\begin{figure}
    \centering
    \includegraphics[width=\linewidth]{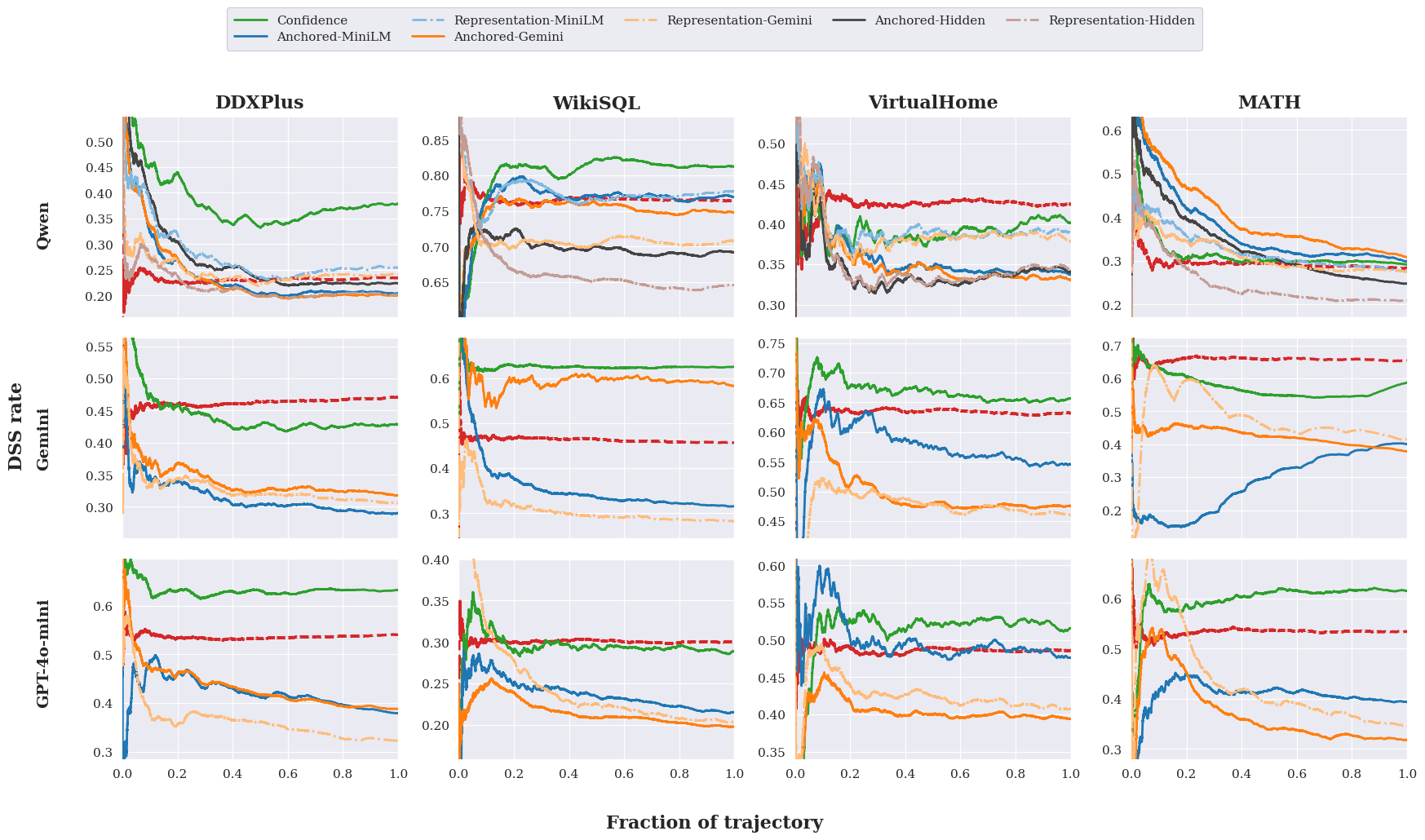}
    \caption{Cumulative support rate $\widehat{\mathrm{SR}}_T = \tfrac{1}{T}\sum_{t=1}^T a_t$ across all task--model pairs, showing every score variant. Rows are base agents, columns are tasks. All variants are run at the same target $\alpha$ as in Figure~\ref{fig:error-control}.}
    \label{app-fig:dss-all-variants}
\end{figure}

\subsection{Ablation on the exploration probability $\mu$}
\label{app:mu-ablation}

Theorem~\ref{thm:finite-sample-control} identifies the exploration probability $\mu$ as the lever that controls the second slack term in the missed-support error bound. larger values of $\mu$ tighten error control and yield smoother convergence, but increase support usage, since the algorithm calls support with probability $\mu$ on every round where the score falls below the threshold. Here we verify this empirically by sweeping $\mu$ on a fixed task--model--score configuration and observing how the missed-support error and the support rate move together.

We fix the base agent to GPT-4o-mini, the task to DDXPlus, and the score to Anchored-Gemini. The threshold step size $\eta_t$ and initial threshold $\lambda_0$ are held constant and only $\mu$ is varied. The configurations are listed in Table~\ref{tab:mu_ablation}.

Figure~\ref{app-fig:mu-ablation} reports the cumulative missed-support error and the cumulative support rate across the five values of $\mu$. The two panels show the predicted tradeoff. On the error-control side, larger values of $\mu$ produce smoother trajectories that converge faster to the target $\alpha$, while small values exhibit noisier adaptation. On the support-efficiency side, small $\mu$ delivers substantially lower support rates, while large $\mu$ pushes the support rate up because the algorithm queries support with non-negligible probability even on rounds where the score is decisively low. In our main experiments we therefore choose moderate values of $\mu$ that achieve reliable error control without sacrificing the support-rate gains.

\begin{figure}[htp]
    \centering
    \includegraphics[width=\linewidth]{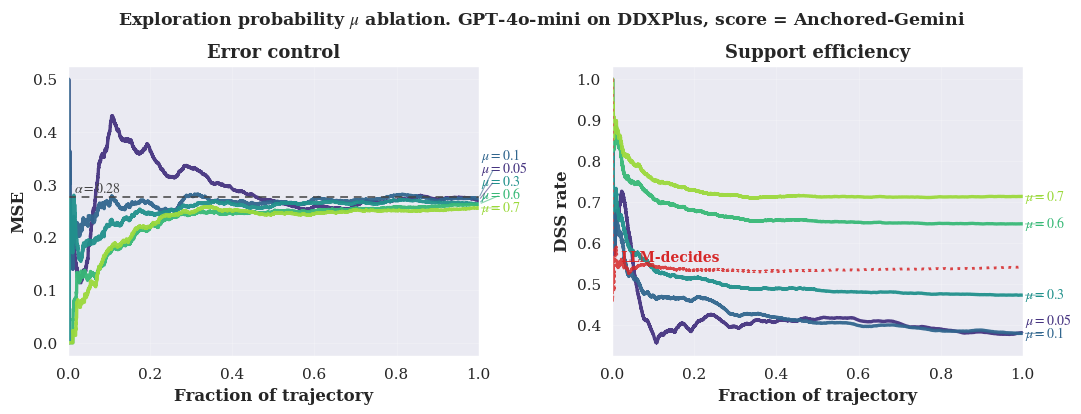}
    \caption{Effect of the exploration probability $\mu$. Base agent: GPT-4o-mini. Task: DDXPlus. Score: Anchored-Gemini. Left: cumulative missed-support error against the target $\alpha$. Right: cumulative support rate, with the \textsc{LLM-Decides} baseline shown for reference. Larger $\mu$ tightens error control and yields smoother convergence but increases support usage, matching the dependence on $\mu$ in the slack term of Theorem~\ref{thm:finite-sample-control}.}
    \label{app-fig:mu-ablation}
\end{figure}

\begin{table}[htp]
\centering\small
\setlength{\tabcolsep}{8pt}
\renewcommand{\arraystretch}{1.15}
\begin{tabular}{@{}ccc@{}}
\toprule
$\mu$ & $\eta_\lambda$ & $\lambda_0$ \\
\midrule
0.05 & 0.015 & 0.5 \\
0.10 & 0.015 & 0.5 \\
0.30 & 0.015 & 0.5 \\
0.60 & 0.015 & 0.5 \\
0.70 & 0.015 & 0.5 \\
\bottomrule
\end{tabular}
\caption{Hyperparameters for the exploration-probability ablation. Base agent: GPT-4o-mini. Task: DDXPlus. Score: Anchored-Gemini. $\mu$ is varied while $\eta_t$ and $\lambda_0$ are held at the base values from the main experiments.}
\end{table}
\label{tab:mu_ablation}
\subsection{Sensitivity to the choice of $g$ on VirtualHome}
\label{app:vh-gain-defs}

The benefit function $g$ encodes what it means for support to materially help on a given task, and its definition naturally depends on the application and on what the provider considers a meaningful improvement. The purpose of this section is to verify that our oversight layer is agnostic to this choice: error control holds for any $g$, and the qualitative ranking of score variants is preserved across different definitions.

We illustrate this on VirtualHome, which is the most since the output is a structured plan rather than a verifiable answer. Each generated plan is evaluated against the gold plan using its longest common subsequence (LCS) score, which measures how much of the gold plan's ordered structure the generated plan recovers. We compare two natural gain definitions built on top of this score:

\begin{itemize}
    \item $g_\mathrm{strict}$ marks support as beneficial when the supported plan reaches an absolute quality bar, $y_1^\mathrm{LCS} \geq 0.5$. This captures the view that improvement matters only if the supported plan crosses a fixed threshold of acceptability.
    \item $g_\mathrm{improved}$ marks support as beneficial whenever the supported plan strictly improves on the unsupported plan in LCS score. This is the most permissive definition and treats any strict improvement as material.
\end{itemize}

The two definitions impose qualitatively different demands. The first asks whether $y_1$ is good in absolute terms, the second asks whether it is better than $y_0$. They induce different positive rates and therefore different targets: the missed-support rate of \textsc{LLM-Decides} is $\alpha = 0.51$ under $g_\mathrm{strict}$ and $\alpha = 0.34$ under $g_\mathrm{improved}$, and we set the target accordingly so all comparisons remain at matched error.

Figure~\ref{app-fig:vh-gain-defs} reports both the cumulative missed-support error (top row) and the cumulative support rate (bottom row) for Gemini-2.5-Flash on VirtualHome under both gain definitions. Two observations stand out. First, error control holds in every panel. The running missed-support error converges to the corresponding target $\alpha$ regardless of which $g$ is used, consistent with Theorem~\ref{thm:finite-sample-control}. Second, the qualitative ranking of methods carries over: the adaptive parameterized score variants reduce the support rate substantially relative to \textsc{LLM-Decides}. This confirms that $g$ is a flexible engineering choice the practitioner can tailor to their task and their notion of meaningful improvement, without changing the algorithm or its guarantees.

\begin{figure}[htp]
    \centering
    \includegraphics[width=\linewidth]{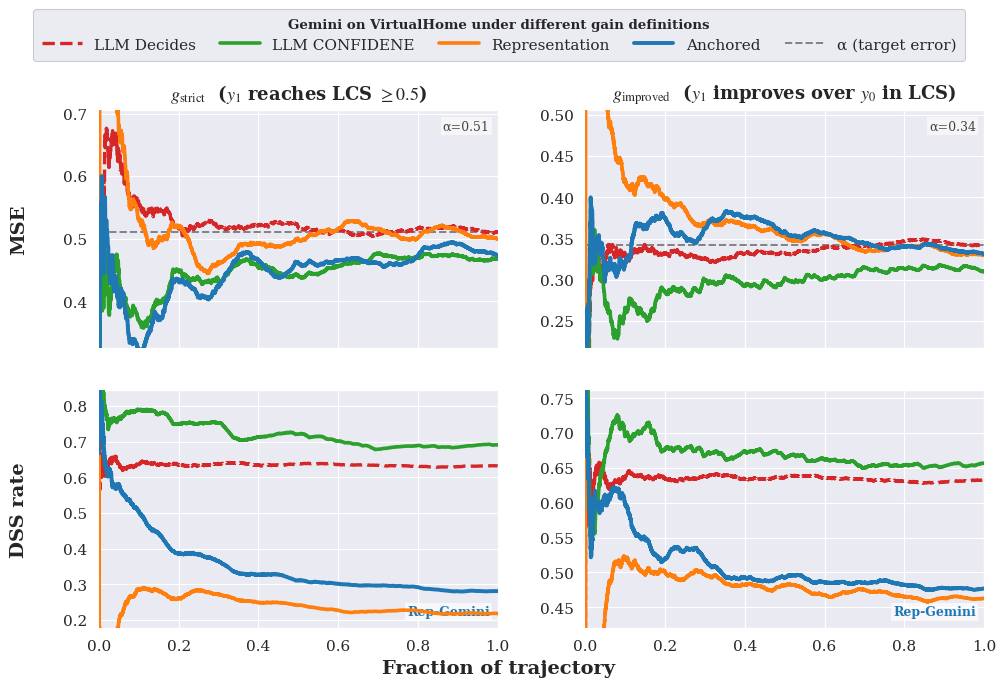}
    \caption{Gemini-2.5-Flash on VirtualHome under two gain definitions. Columns are gain definitions: $g_\mathrm{strict}$ ($y_1$ reaches LCS $\geq 0.5$) on the left, $g_\mathrm{improved}$ ($y_1$ improves over $y_0$ in LCS) on the right. Top row reports the cumulative missed-support error against the corresponding target $\alpha$; bottom row reports the cumulative support rate. Across both definitions, the running error converges to $\alpha$ and the parameterized variants reduce support usage relative to \textsc{LLM-Decides}.}
    \label{app-fig:vh-gain-defs}
\end{figure}

\subsection{Scores histogram across the online stream}

To complement the support-rate results in Section~\ref{sec:exp}, we visualize the score distributions produced by each score family along the online stream. Figure~\ref{app-fig:score-histograms} shows histograms of the score $s$ split by the latent benefit variable $g$ for Gemini-2.5-Flash on all four tasks. The three rows correspond to the three score families introduced in Section~\ref{subsec:score-design}: the raw LLM confidence $\hat{g}_\mathrm{LLM}$, the Anchored score that adds a calibrated residual on top of $\hat{g}_\mathrm{LLM}$, and the Representation score that drops the anchor and learns from the embedding alone. The Anchored and Representation variants both use the \textsc{Gemini-Embedding-2} model  ($\phi = $ Gemini) and are calibrated on the fly.

The top row makes the limitation of relying on raw confidence concrete. On WikiSQL and MATH, both $g = 0$ and $g = 1$ pile up at the similar scores, so LLM Confidence carries almost no information about whether support would help; on DDXPlus and VirtualHome the two distributions overlap heavily as well. A threshold rule applied directly to this score is therefore forced to either accept too many $g=0$ instances or call support on too many $g=1$ instances, which is why Confidence ends up close to or above \textsc{LLM-Decides} in support efficiency.

The middle and bottom rows show what calibration-on-the-fly recovers from the same stream. The Anchored score sharpens the confidence signal where it carries useful information (DDXPlus, VirtualHome) and corrects it where it does not (WikiSQL, MATH), pushing the $g=1$ mass toward high scores and the $g=0$ mass toward low scores. The Representation score, which has no anchor to start from, learns this separation directly from the embedding. This is the mechanism that makes the parameterized score families consistently outperform Confidence in Figure~\ref{fig:dss-main}. Even when the underlying signal is uninformative or misleading, the calibration-on-the-fly extracts a usable separation between beneficial and non-beneficial instances from the feedback received on rounds where support is invoked.

\begin{figure}[htp]
    \centering
    \includegraphics[width=\linewidth]{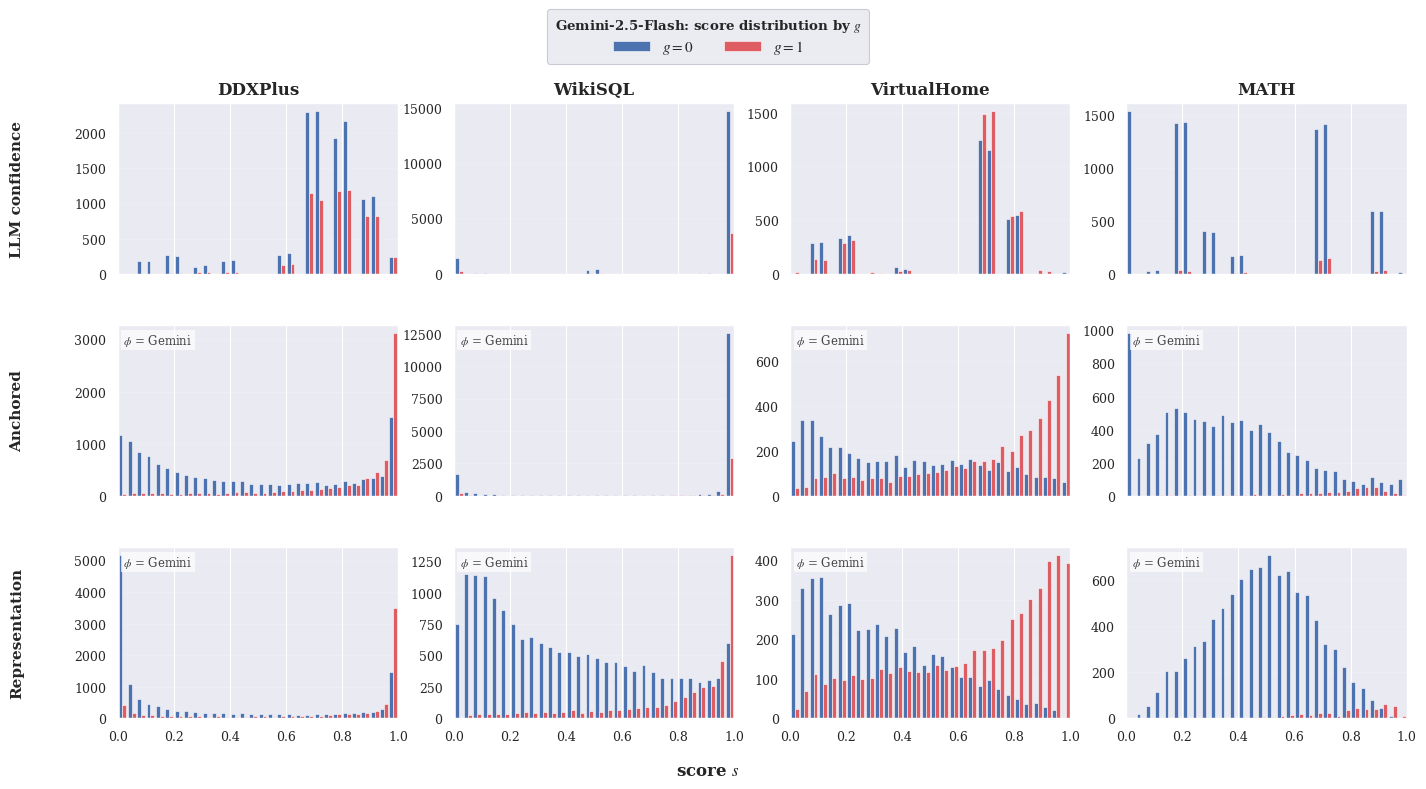}
    \caption{Score distributions along the online stream for Gemini-2.5-Flash, split by the latent benefit variable $g$. Rows are score families (LLM confidence, Anchored, Representation) and columns are tasks (DDXPlus, WikiSQL, VirtualHome, MATH). Anchored and Representation variants both use the Gemini embedding ($\phi = $ Gemini) and are calibrated on the fly. Raw confidence often fails to separate $g=0$ from $g=1$, particularly on WikiSQL and MATH where both classes concentrate at the same scores. After calibration-on-the-fly, the Anchored and Representation scores recover separation between the two classes, providing the threshold rule with a more informative signal.}
    \label{app-fig:score-histograms}
\end{figure}

\subsection{Per-task specifics and descriptive statistics}
\label{app:per-task-statistics}
Table~\ref{app-tab:per_task_stats} reports descriptive statistics for each task--model pair. We include the accuracy of the unsupported output $y_0$, the accuracy of the supported output $y_1$, the empirical probability $\Pr(g=1)$ that support is beneficial, and the support rate and missed-support error of the \textsc{LLM-Decides} baseline. These quantities highlight that the value of support depends jointly on the task and the base model: even within the same task, support can be far more useful for one model than another. For example, on DDXPlus, Gemini-2.5-Flash benefits substantially more from the additional clinical evidence than Qwen-2.5-7B, despite both receiving identical support. The headroom for support to help is therefore a property of the task, model, and support modality together, rather than of the task alone. This variability is precisely what motivates the oversight layer we develop in the paper. Rather than relying on the agent's own support-seeking behavior, which can drift far from $\Pr(g=1)$ in either direction, our algorithm adapts online to the specific task--model--support combination at hand and controls the missed-support error at a user-chosen level.
\begin{table}[htp]

\centering\small
\setlength{\tabcolsep}{6pt}
\renewcommand{\arraystretch}{1.15}
\begin{tabular}{@{}llccccc@{}}
\toprule
Model & Task & $y_0$ accuracy & $y_1$ accuracy & $\Pr(g=1)$ & LLM DSS rate & LLM MSE \\
\midrule
\multirow{4}{*}{\textbf{Qwen-2.5-7B}} & DDXPlus & 37.4\% & 36.9\% & 0.106 & 0.234 & 0.708 \\
 & WikiSQL & 60.2\% & 74.4\% & 0.240 & 0.764 & 0.108 \\
 & VirtualHome & 8.4\% & 9.6\% & 0.416 & 0.424 & 0.577 \\
 & MATH & 57.4\% & 66.6\% & 0.168 & 0.284 & 0.714 \\
\midrule
\multirow{4}{*}{\textbf{Gemini-2.5-Flash}} & DDXPlus & 49.8\% & 76.5\% & 0.341 & 0.471 & 0.528 \\
 & WikiSQL & 63.4\% & 76.5\% & 0.190 & 0.457 & 0.425 \\
 & VirtualHome & 13.6\% & 20.0\% & 0.514 & 0.632 & 0.342 \\
 & MATH & 90.8\% & 92.2\% & 0.042 & 0.654 & 0.143 \\
\midrule
\multirow{4}{*}{\textbf{GPT-4o-mini}} & DDXPlus & 44.1\% & 54.7\% & 0.203 & 0.540 & 0.276 \\
 & WikiSQL & 70.8\% & 79.3\% & 0.167 & 0.300 & 0.633 \\
 & VirtualHome & 15.0\% & 25.8\% & 0.614 & 0.486 & 0.463 \\
 & MATH & 60.6\% & 68.4\% & 0.130 & 0.534 & 0.369 \\
\bottomrule
\end{tabular}
\caption{Per (task,model) descriptive statistics for each base agent.}
\label{app-tab:per_task_stats}
\end{table}
\subsection{Error Control for all datasets}
\label{app:error-control-all}

Figure~\ref{app-fig:error-control-all} reports the cumulative missed-support error $\mathrm{MSE}(T)$ across the full $3 \times 4$ grid of base agents and tasks. The main text shows convergence for a single representative pair (Figure~\ref{fig:error-control}); here we verify that the same behavior holds in every panel. Across all twelve task--model pairs and across every score variant, the running MSE converges to the target level $\alpha$, set to the rate achieved by \textsc{LLM-Decides} on that pair. This confirms that the finite-sample guarantee of Theorem~\ref{thm:finite-sample-control} holds uniformly across tasks, models, and score families. The bound in that theorem decomposes into two qualitatively different sources of slack. The first is the intrinsic error of online quantile tracking, governed by the threshold step size $\eta$: larger values facilitate faster initial convergence but produce noisier trajectories, while smaller values yield smoother curves at the cost of slower adaptation. The second arises from partial feedback and the randomized exploration needed to obtain unbiased feedback, and is controlled by the exploration probability $\mu$: larger $\mu$ tightens error control but increases support usage. We study the dependence on $\mu$ empirically in Appendix~\ref{app:mu-ablation}.

\begin{figure}[htp]
    \centering
    \includegraphics[width=\linewidth]{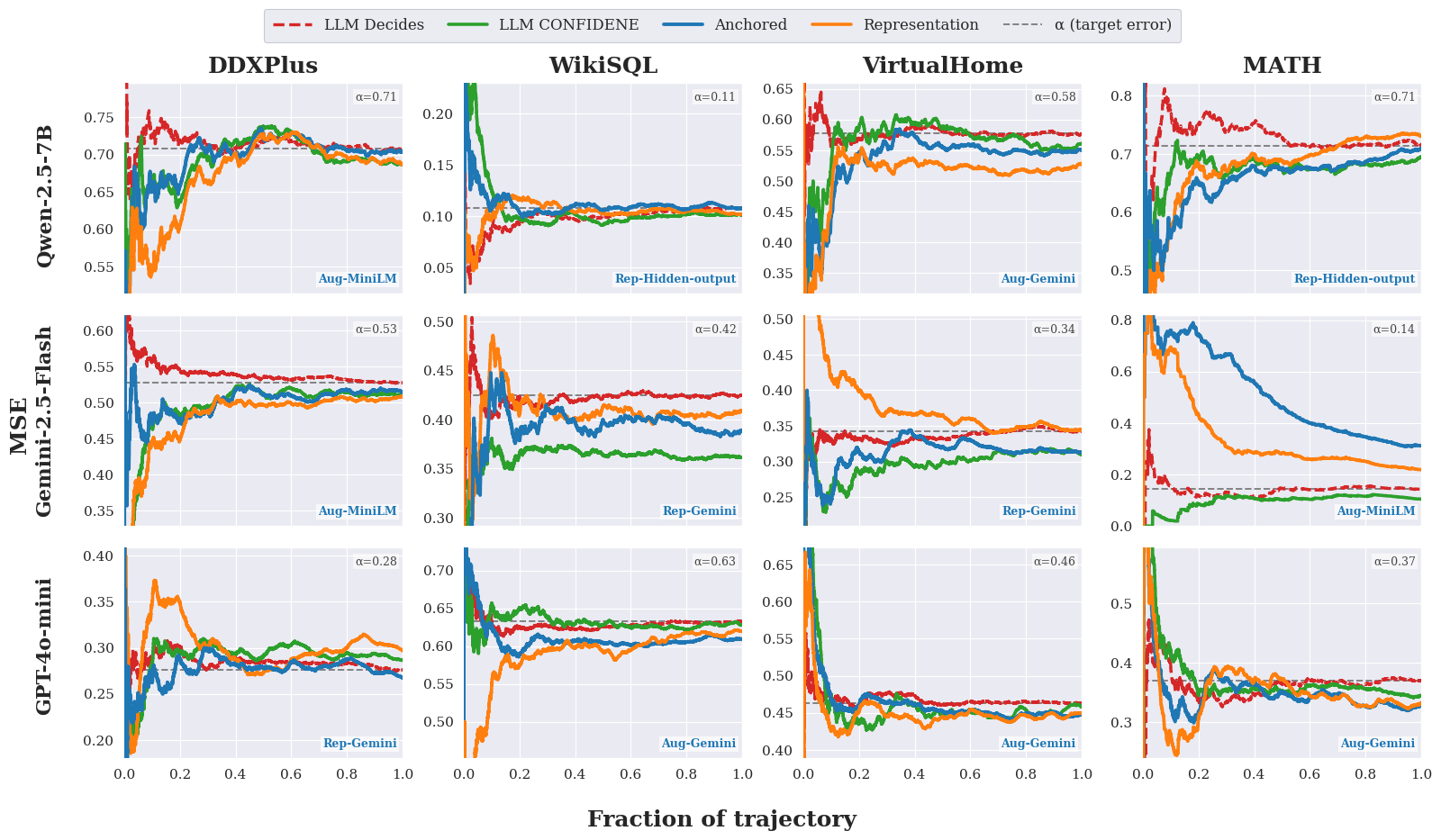}
    \caption{Cumulative missed-support error $\mathrm{MSE}(T)$ across all task--model pairs. Rows are base agents (Qwen-2.5-7B, Gemini-2.5-Flash, GPT-4o-mini), columns are tasks (DDXPlus, WikiSQL, VirtualHome, MATH). Each panel shows the running MSE for all score variants together with the \textsc{LLM-Decides} baseline and the target level $\alpha$. All variants converge towards $\alpha$ regardless of the score family.}
    \label{app-fig:error-control-all}
\end{figure}

\subsection{Hyperparameters}
\label{app:hyperparameters}
For completeness, Table~\ref{tab:hp_all} reports the full hyperparameter configuration used to produce the main-text results across all base agents, tasks, and score variants. We report the calibration step size $\eta_t$, the threshold step size $\gamma_t$, the exploration probability $\mu$, and the initial threshold $\lambda_0$. Confidence has no learnable parameters and so its $\eta_t$ entry is left blank. Hidden-state variants are only available in the white-box regime and are reported for Qwen-2.5-7B; the corresponding rows for Gemini-2.5-Flash and GPT-4o-mini are left blank.

The values reported in Table~\ref{tab:hp_all} were selected by a simple grid search over a small set of candidate values on a held out test stream for each hyperparameter, with the support rate at the end of the trajectory used as the selection criterion. We did not find this search to be delicate. Across all task--model pairs, a wide range of reasonable choices yield similar performance, and the qualitative conclusions of the paper are insensitive to the exact configuration. In practice, the threshold step size $\eta_t$ governs how quickly the algorithm adapts; the exploration probability $\mu$ trades off tighter error control against higher support usage and is studied directly in Appendix~\ref{app:mu-ablation}; and the score update step size $\gamma_t$ and initial threshold $\lambda_0$ have only secondary effects once the other two are set sensibly. We therefore expect a practitioner to obtain comparable results without an extensive hyperparameter search.

\begin{table}[htp]
\centering\scriptsize
\setlength{\tabcolsep}{2.5pt}
\renewcommand{\arraystretch}{1.1}
\resizebox{\textwidth}{!}{%
\begin{tabular}{@{}ll|cccc|cccc|cccc|cccc@{}}
\toprule
Model & Variant & \multicolumn{4}{c}{\textbf{DDXPlus}} & \multicolumn{4}{c}{\textbf{WikiSQL}} & \multicolumn{4}{c}{\textbf{VirtualHome}} & \multicolumn{4}{c}{\textbf{MATH}} \\
\cmidrule(lr){3-6} \cmidrule(lr){7-10} \cmidrule(lr){11-14} \cmidrule(lr){15-18}
 & & $\gamma_t$ & $\eta_t$ & $\mu$ & $\lambda_0$ & $\gamma_t$ & $\eta_t$ & $\mu$ & $\lambda_0$ & $\gamma_t$ & $\eta_t$ & $\mu$ & $\lambda_0$ & $\gamma_t$ & $\eta_t$ & $\mu$ & $\lambda_0$ \\
\midrule
\multirow{7}{*}{\textbf{Qwen-2.5-7B}} & Confidence & --- & 0.05 & 0.1 & 0.4 & --- & 0.01 & 0.2 & 0.3 & --- & 0.05 & 0.1 & 0.4 & --- & 0.1 & 0.25 & 0.6 \\
 & Anchored-MiniLM & 2 & 0.05 & 0.1 & 0.5 & 2 & 0.01 & 0.2 & 0.3 & 2 & 0.05 & 0.1 & 0.5 & 2 & 0.01 & 0.1 & 0.6 \\
 & Representation-MiniLM & 10 & 0.05 & 0.1 & 0.5 & 2 & 0.01 & 0.2 & 0.3 & 10 & 0.05 & 0.1 & 0.5 & 2 & 0.01 & 0.1 & 0.6 \\
 & Anchored-Gemini & 2 & 0.05 & 0.1 & 0.5 & 2 & 0.01 & 0.2 & 0.3 & 2 & 0.05 & 0.1 & 0.5 & 2 & 0.01 & 0.1 & 0.6 \\
 & Representation-Gemini & 5 & 0.05 & 0.1 & 0.5 & 2 & 0.01 & 0.2 & 0.3 & 10 & 0.05 & 0.1 & 0.5 & 2 & 0.01 & 0.1 & 0.6 \\
 & Anchored-Hidden & 2 & 0.05 & 0.1 & 0.5 & 2 & 0.01 & 0.2 & 0.3 & 0.5 & 0.05 & 0.1 & 0.5 & 5 & 0.01 & 0.1 & 0.6 \\
 & Representation-Hidden & 2 & 0.05 & 0.1 & 0.5 & 2 & 0.01 & 0.2 & 0.3 & 2 & 0.05 & 0.1 & 0.5 & 5 & 0.01 & 0.1 & 0.6 \\
\midrule
\multirow{7}{*}{\textbf{Gemini-2.5-Flash}} & Confidence & --- & 0.01 & 0.1 & 0.5 & --- & 0.001 & 0.3 & 0.99 & --- & 0.01 & 0.05 & 0.6 & --- & 0.01 & 0.1 & 0.2 \\
 & Anchored-MiniLM & 2 & 0.05 & 0.1 & 0.5 & 2 & 0.05 & 0.1 & 0.5 & 1 & 0.05 & 0.05 & 0.5 & 1 & 0.01 & 0.1 & 0.7 \\
 & Representation-MiniLM & 2 & 0.01 & 0.2 & 0.3 & 2 & 0.01 & 0.2 & 0.3 & 2 & 0.01 & 0.2 & 0.3 & 2 & 0.01 & 0.2 & 0.3 \\
 & Anchored-Gemini & 5 & 0.05 & 0.1 & 0.5 & 5 & 0.05 & 0.1 & 0.5 & 10 & 0.01 & 0.05 & 0.6 & 5 & 0.01 & 0.1 & 0.6 \\
 & Representation-Gemini & 20 & 0.05 & 0.1 & 0.5 & 5 & 0.05 & 0.1 & 0.5 & 10 & 0.01 & 0.05 & 0.9 & 5 & 0.05 & 0.1 & 0.65 \\
 & Anchored-Hidden & --- & --- & --- & --- & --- & --- & --- & --- & --- & --- & --- & --- & --- & --- & --- & --- \\
 & Representation-Hidden & --- & --- & --- & --- & --- & --- & --- & --- & --- & --- & --- & --- & --- & --- & --- & --- \\
\midrule
\multirow{7}{*}{\textbf{GPT-4o-mini}} & Confidence & --- & 0.01 & 0.1 & 0.5 & --- & 0.1 & 0.1 & 0.5 & --- & 0.01 & 0.05 & 0.5 & --- & 0.05 & 0.4 & 0.5 \\
 & Anchored-MiniLM & 2 & 0.015 & 0.1 & 0.5 & 2 & 0.08 & 0.1 & 0.7 & 1 & 0.05 & 0.05 & 0.7 & 10 & 0.1 & 0.25 & 0.8 \\
 & Representation-MiniLM & 2 & 0.01 & 0.2 & 0.3 & 2 & 0.01 & 0.2 & 0.3 & 2 & 0.01 & 0.2 & 0.3 & 2 & 0.01 & 0.2 & 0.3 \\
 & Anchored-Gemini & 5 & 0.015 & 0.1 & 0.5 & 5 & 0.005 & 0.1 & 0.5 & 10 & 0.01 & 0.05 & 0.5 & 20 & 0.05 & 0.1 & 0.8 \\
 & Representation-Gemini & 20 & 0.01 & 0.05 & 0.5 & 10 & 0.01 & 0.05 & 0.5 & 10 & 0.01 & 0.05 & 0.5 & 5 & 0.05 & 0.1 & 0.8 \\
 & Anchored-Hidden & --- & --- & --- & --- & --- & --- & --- & --- & --- & --- & --- & --- & --- & --- & --- & --- \\
 & Representation-Hidden & --- & --- & --- & --- & --- & --- & --- & --- & --- & --- & --- & --- & --- & --- & --- & --- \\
\bottomrule
\end{tabular}}
\caption{Hyperparameters across all base agents and tasks. $\gamma_t$ is the calibration-on-the-fly step size, $\eta_t$ the threshold step size, $\mu$ the exploration probability, $\lambda_0$ the initial threshold. Confidence has no learnable parameters; its $\eta_t$ cell is left blank.}
\label{tab:hp_all}
\end{table}

\subsection{Task implementation details and example prompts}
\label{app:prompts}

We instantiate each of the four tasks from Section~\ref{sec:exp}. The base agent producing $y_0$ and $y_1$ is one of Qwen-2.5-7B, Gemini-2.5-Flash, or GPT-4o-mini, identical across all four tasks. The support modality varies by task: revealed follow-up questions, examination findings, and laboratory results on DDXPlus; SQL query formulation and execution against the table on WikiSQL; scene-specific object locations and constraints on VirtualHome; and targeted guidance from a stronger reasoner on MATH. For MATH, we use DeepSeek as the stronger reasoner that responds to the base agent's questions; the base agent itself remains one of the three models above and is responsible for both the initial solution $y_0$ and the revised solution $y_1$ that incorporates the guidance.

For concreteness, we provide the full set of prompts used in the MATH pipeline below. They cover different stages of the interaction: producing the initial solution $y_0$, eliciting the self-reported confidence $\hat{g}_\mathrm{LLM}$ used as the black-box anchor signal in Section~\ref{subsec:score-design}, deciding whether to seek support (used by the \textsc{LLM-Decides} baseline), formulating targeted questions for the stronger reasoner, generating the expert guidance, and producing the revised solution $y_1$. The prompts for the other three tasks follow the same overall structure, adapted to the task-specific support modality, and will be released alongside the code upon acceptance.
\begin{promptbox}{Weak Solver --- Initial Solution $y_0$ and Self-Reported Confidence $\hat{g}_{\mathrm{LLM}}$ (MATH)}
\label{prompt:math-y0-ghat}
\textbf{System:} You are a math competition solver and calibration assistant. \\
\\
\textbf{User:} Solve the following problem and assess your need for expert support. \\
\\
\textbf{Problem:} \{\{problem\}\} \\
\\
Respond in the following JSON format only, with no additional text: \\
\\
\texttt{\{} \\
\texttt{\ \ "solution": "<step-by-step reasoning, with the final answer in \textbackslash boxed\{\}>",} \\
\texttt{\ \ "confidence": <number between 0.0 and 1.0>} \\
\texttt{\}} \\
\\
For the \texttt{confidence} field: if you could ask a math expert specific questions about this problem, how likely is it that their response would lead you to a different, correct answer? Consider whether you were guessing, whether any step is unsure, and whether knowing the right technique would change your answer. Higher values mean expert input is more likely to change your answer.
\end{promptbox}

\begin{promptbox}{Weak Solver --- Decision to Seek Support (MATH)}
\label{prompt:math-decide}
\textbf{System:} You are a careful math student. If there is any doubt about your work, it is worth asking. \\
\\
\textbf{User:} Problem: \{\{problem\}\} \\
Your work: \\
\{\{y0\_raw\}\} \\
Your answer: \{\{y0\_answer\}\} \\
\\
You can ask a math expert questions about this problem. Their guidance could help you find errors or discover a better approach. Asking costs time but can prevent submitting a wrong answer. Thus you should ask if you think it will improve your final answer, and you should proceed with the current answer if you think it is unlikely that expert guidance will change and improve your answer. \\
\\
Consider: Are you certain your approach is correct, or were there steps where you guessed or felt unsure? \\
\\
Respond with ONLY one of: \\
\textbf{REQUEST} --- if any step felt uncertain \\
\textbf{PROCEED} --- if you are confident in every step
\end{promptbox}

\begin{promptbox}{Weak Solver --- Question Formulation (MATH)}
\label{prompt:math-questions}
\textbf{System:} You are a math student asking your tutor for help. Ask specific, targeted questions about parts you're stuck on. Do NOT ask for the full solution or the final answer. \\
\\
\textbf{User:} You attempted this problem and got \{\{y0\_answer\}\}: \\
\\
\{\{problem\}\} \\
\\
Your work: \\
\{\{y0\_raw\}\} \\
\\
You can ask a math expert up to 3 questions about this problem. They will help you without solving the entire problem for you. \\
\\
Ask about the specific steps or concepts you're most unsure about. For example: \\
- ``I used [method] --- is this the right approach?'' \\
- ``I got [result] at this step --- is this correct?'' \\
- ``How should I handle [specific part]?'' \\
- ``What identity or theorem applies to [this expression]?'' \\
\\
Write your questions:
\end{promptbox}

\begin{promptbox}{Strong Expert --- Targeted Guidance (MATH)}
\label{prompt:math-expert}
\textbf{System:} You are a math tutor. A student is working on a problem and has asked you specific questions. Be maximally helpful --- give clear, direct, and detailed answers to each of their questions. If they made an error, point it out and explain why it's wrong. If they're using the wrong approach, tell them the right one and explain the key first step. If they ask about a technique, explain it concretely with enough detail that they can apply it. \\
\\
The ONE thing you must NOT do is solve the entire problem for them or state the final answer. Help them get unstuck, but let them finish the last steps on their own. \\
\\
\textbf{User:} Problem: \\
\{\{problem\}\} \\
\\
Student's questions: \\
\{\{student\_questions\}\}
\end{promptbox}

\begin{promptbox}{Weak Solver --- Revised Solution $y_1$ with Guidance (MATH)}
\label{prompt:math-y1}
\textbf{System:} You are a math competition solver. You previously attempted this problem and got stuck on a specific step, so you asked an expert for guidance. Now solve the problem from scratch using their advice. Do NOT repeat your previous mistakes. Put your final answer in \texttt{\textbackslash boxed\{\}}. \\
\\
\textbf{User:} Problem: \\
\{\{problem\}\} \\
\\
The question you asked the expert: \\
\{\{expert\_question\}\} \\
\\
Expert guidance: \\
\{\{expert\_answer\}\} \\
\\
Solve the problem using this guidance. Put your final answer in \texttt{\textbackslash boxed\{\}}.
\end{promptbox}

\subsection{Operational Variants for the Score Input}
\label{app:operational}

The support decision can be made on the basis of any information available at decision time. We outline three natural choices, all of which fit within the framework of Section~\ref{sec:fundamentals} and inherit the guarantee of Theorem~\ref{thm:finite-sample-control}.

\paragraph{Score on $x$ alone.} The score is computed from the input directly, $s = s_\theta(\phi(x))$, before any output is generated. This is the cheapest option, and is appropriate when $y_0$ is expensive to produce relative to the score itself, e.g., when it requires a long reasoning chain or an expensive tool call.

\paragraph{Score on $(x, y_0)$.} The agent generates $y_0$ first and then scores $(x, y_0)$. This costs an extra forward pass per round but gives the score strictly more information.

\paragraph{Score on $(x, \tilde{y}_0)$ for a cheap surrogate $\tilde{y}_0$.} When $y_0$ itself is expensive, the score can use a cheap surrogate available at decision time, such as the first few tokens of a generation, an intermediate reasoning trace, or a smaller model's prediction. Two conditions must hold: $\tilde{y}_0$ must be available before the support decision, and the feedback signal $g_t$ must still be computable on rounds where support is sought.

The threshold-update guarantee in Theorem~\ref{thm:finite-sample-control} holds for any of these choices. The choice is purely operational as it sets the cost of running the algorithm and the quality of the score, both of which determine how often support must be invoked to hit a given missed-support error.

Figure~\ref{app-fig:operational-variants} reports the empirical comparison on Gemini-2.5-Flash, DDXPlus, with the Anchored-Gemini score. We include the three variants above as well as two reference points that drop the input from the score: scoring on $y_0$ alone, and scoring on the reasoning trace alone. All five variants control the missed-support error at the target $\alpha$, consistent with Theorem~\ref{thm:finite-sample-control}. On support efficiency, the three variants that include $x$ or the reasoning trace (with or without $y_0$) achieve essentially identical support rates, while the variants that is only based on generated  $y_0$ is noticeably worse. The takeaway is that $x$ or the reasoning traces carry most of the signal needed to predict whether support would be beneficial: adding $y_0$ or the reasoning trace on top of it does not help, but removing $x$ does hurt. From a practical standpoint, this is a useful finding, since scoring on $x$ alone is the cheapest of the variants that retain efficiency: the support decision can be made before any generation is committed. This finding is however task specific- and in practice this might change depending on the underlying task.

Figure~\ref{app-fig:operational-variants} reports the empirical comparison on Gemini-2.5-Flash, DDXPlus, with the Anchored-Gemini score. We include the three variants above as well as two reference points that drop the input from the score: scoring on $y_0$ alone, and scoring on the reasoning trace alone. All five variants control the missed-support error at the target $\alpha$, consistent with Theorem~\ref{thm:finite-sample-control}. On support efficiency, the variants based on $x$ or the reasoning trace (with or without $y_0$) achieve essentially identical support rates, while the variant based on $y_0$ alone is noticeably worse. The takeaway is that $x$ and the reasoning trace each carry most of the signal needed to predict whether support would be beneficial, while $y_0$ alone does not. From a practical standpoint, this is a useful finding, since scoring on $x$ alone is the cheapest of the efficient variants: the support decision can be made before any generation is committed. This finding is, however, task-specific, and the relative informativeness of $x$, $y_0$, and the reasoning trace may differ depending on the complexity of the underlying task in practice. 

\begin{figure}[htp]
    \centering
    \includegraphics[width=\linewidth]{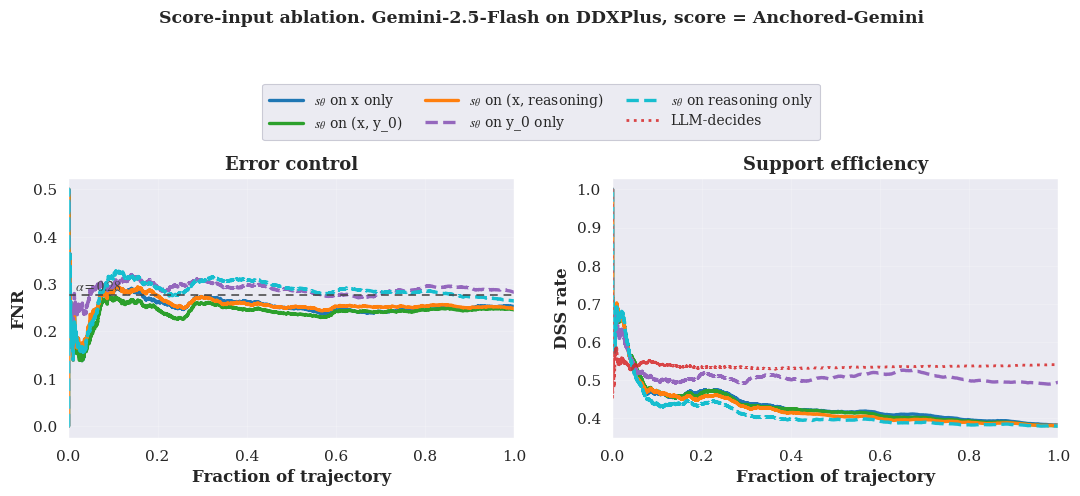}
    \caption{Score-input ablation. Base agent: Gemini-2.5-Flash. Task: DDXPlus. Score: Anchored-Gemini. Left: cumulative missed-support error against the target $\alpha$. Right: cumulative support rate, with the \textsc{LLM-Decides} baseline shown for reference. The five variants score the input alone ($s_\theta$ on $x$), the input together with the unsupported outcome ($s_\theta$ on $(x, y_0)$), the input together with the agent's reasoning trace ($s_\theta$ on $(x, \mathrm{reasoning})$), the unsupported outcome alone ($s_\theta$ on $y_0$), or the reasoning trace alone ($s_\theta$ on $\mathrm{reasoning}$). All five variants control the missed-support error at $\alpha$. Variants that include $x$ achieve similar support rates; dropping $x$ degrades efficiency.}
    \label{app-fig:operational-variants}
\end{figure}

\end{document}